\documentclass{article}

\usepackage{microtype}
\usepackage{graphicx}
\usepackage{subfigure}
\usepackage{booktabs} % for professional tables
\usepackage{amsthm}
\usepackage{amsmath}
\usepackage{amssymb}
\usepackage{thmtools}
\newif\ifcomment
\commenttrue
\usepackage{hyperref}

\usepackage[accepted]{icml2020}

\icmltitlerunning{An investigation of why overparameterization exacerbates spurious correlations}

\newcommand{\pmajmath}{p_\mathsf{maj}}
\newcommand{\scrmath}{r_\mathsf{s:c}}

\newcommand{\wg}{worst-group }

\newtheorem{lemma}{Lemma}

\newcommand{\relu}{\text{ReLU}}
\newcommand{\nmin}{n_\mathsf{min}}
\newcommand{\nmaj}{n_\mathsf{maj}}
\newcommand{\nnoise}{N}
\newcommand{\nnot}{N_0}
\newcommand{\ntot}{n}
\newcommand{\gmaj}{G_\mathsf{maj}}
\newcommand{\gmin}{G_\mathsf{min}}

\newcommand{\xcau}{x_\mathsf{core}}
\newcommand{\xspu}{x_\mathsf{spu}}
\newcommand{\xnoise}{x_\mathsf{noise}}
\newcommand{\sigcau}{\sigma_\mathsf{core}}
\newcommand{\sigspu}{\sigma_\mathsf{spu}}
\newcommand{\signoise}{\sigma_\mathsf{noise}}

\newcommand{\esterm}{{\hat{w}^\mathsf{erm}}}
\newcommand{\estrw}{\hat{w}^\mathsf{rw}}
\newcommand{\estrwl}{\hat{w}^\mathsf{rw}_\lambda}
\newcommand{\estmm}{{\hat{w}^\mathsf{minnorm}}}
\newcommand{\estmmprimal}{{\hat{w}^\mathsf{mm}}}
\newcommand\roberr[1]{\text{Err}_{\mathsf{wg}}({#1})}

\newcommand{\memmin}{w^\mathsf{use-spu}}
\newcommand{\memmaj}{w^\mathsf{use-core}}
\newcommand{\memmajmm}{{\bar{w}}^\mathsf{use-core}}
\newcommand{\memall}{w^\mathsf{use-core}}
\newcommand{\Memmin}{\sW^\mathsf{use-spu}}
\newcommand{\Memmaj}{\sW^\mathsf{use-core}}

\newcommand{\textmin}{\mathsf{use-spu}}
\newcommand{\cau}[1]{#1_\mathsf{core}}
\newcommand{\spu}[1]{#1_\mathsf{spu}}
\newcommand{\noise}[1]{#1_\mathsf{noise}}

\newcommand\ixspu[1]{{\xspu^{(#1)}}}
\newcommand\ixnoi[1]{{\xnoise^{(#1)}}}

\newcommand\ix[1]{x^{(#1)}}
\newcommand\iy[1]{y^{(#1)}}
\newcommand\ig[1]{g^{(#1)}}
\newcommand\is[2]{\alpha^{(#2)}(#1)}
\newcommand\isw[2]{\alpha^{(#2)}}
\newcommand\ellrw{\ell_\textsf{rw}}
\newcommand\Lrw{\text{L}_\textsf{rw}}
\newcommand\Lrwmaj{\text{L}_\textsf{rw-maj}}
\newcommand\Lrwmin{\text{L}_\textsf{rw-min}}

\newcommand\dss{\text{D}^\textsf{ss}}
\newcommand\dgss{\text{D}^\textsf{ss}_g}

\newcommand\deltrain[1]{\delta_{\text{maj-train}}\left(#1\right)}
\newcommand\deltrainall[1]{\delta_{\text{train}}\left(#1\right)}
\newcommand\phiinv{\Phi^{-1}}

\newcommand\underlinepara[1]{\paragraph*{\underline{\underline{#1}}}}

\newlength{\widebarargwidth}
\newlength{\widebarargheight}
\newlength{\widebarargdepth}

\newcommand\sA{\ensuremath{\mathcal{A}}}

\newcommand\sG{\ensuremath{\mathcal{G}}}

\newcommand\sN{\ensuremath{\mathcal{N}}}

\newcommand\sW{\ensuremath{\mathcal{W}}}

\newcommand\sY{\ensuremath{\mathcal{Y}}}

\newcommand\BP{\ensuremath{\mathbb{P}}}

\DeclareMathOperator*{\cov}{Cov} % Covariance
\DeclareMathOperator*{\diag}{diag} % Diagonal matrix
\newcommand\R{\ensuremath{\mathbb{R}}} % Real numbers
\newcommand\eqdef{\ensuremath{\stackrel{\rm def}{=}}} % Equal by definition
\newcommand\refeqn[1]{(\ref{eqn:#1})}

\newcommand\refsec[1]{Section~\ref{sec:#1}}

\newcommand\reffig[1]{Figure~\ref{fig:#1}}

\newcommand\refapp[1]{Appendix~\ref{sec:#1}}
\newcommand\refthm[1]{Theorem~\ref{thm:#1}}

\newcommand\reflem[1]{Lemma~\ref{lem:#1}}

\newcommand\refprop[1]{Proposition~\ref{prop:#1}}
\newcommand\refdef[1]{Definition~\ref{def:#1}}
\newcommand\refcor[1]{Corollary~\ref{cor:#1}}

\ifthenelse{\isundefined{\definition}}{\newtheorem{definition}{Definition}}{}
\ifthenelse{\isundefined{\assumption}}{\newtheorem{assumption}{Assumption}}{}
\ifthenelse{\isundefined{\hypothesis}}{\newtheorem{hypothesis}{Hypothesis}}{}
\ifthenelse{\isundefined{\proposition}}{\newtheorem{proposition}{Proposition}}{}
\ifthenelse{\isundefined{\theorem}}{\newtheorem{theorem}{Theorem}}{}
\ifthenelse{\isundefined{\lemma}}{\newtheorem{lemma}{Lemma}}{}
\ifthenelse{\isundefined{\corollary}}{\newtheorem{corollary}{Corollary}}{}
\ifthenelse{\isundefined{\alg}}{\newtheorem{alg}{Algorithm}}{}
\ifthenelse{\isundefined{\example}}{\newtheorem{example}{Example}}{}
\newcommand{\E}{\ensuremath{\mathbb{E}}} % Expectation
\begin{document}

\twocolumn[
\icmltitle{An investigation of why overparameterization exacerbates spurious correlations}



\icmlsetsymbol{equal}{*}

\begin{icmlauthorlist}
\icmlauthor{Shiori Sagawa}{equal,stan}
\icmlauthor{Aditi Raghunathan}{equal,stan}
\icmlauthor{Pang Wei Koh}{equal,stan}
\icmlauthor{Percy Liang}{stan}
\end{icmlauthorlist}

\icmlaffiliation{stan}{Stanford University}

\icmlcorrespondingauthor{Shiori Sagawa}{\mbox{ssagawa@cs.stanford.edu}}
\icmlcorrespondingauthor{Aditi Raghunathan}{\mbox{aditir@stanford.edu}}
\icmlcorrespondingauthor{Pang Wei Koh}{\mbox{pangwei@cs.stanford.edu}}
\icmlkeywords{}

\vskip 0.3in
]



\printAffiliationsAndNotice{\icmlEqualContribution} % otherwise use the standard text.

\begin{abstract}
We study why overparameterization---increasing model size well beyond the point of zero training error---can hurt test error on minority groups despite improving average test error when there are spurious correlations in the data.
Through simulations and experiments on two image datasets, we identify two key properties of the training data that drive this behavior: the proportions of majority versus minority groups, and the signal-to-noise ratio of the spurious correlations.
We then analyze a linear setting and theoretically show how the inductive bias of models towards ``memorizing'' fewer examples can cause overparameterization to hurt.
Our analysis leads to a counterintuitive approach of subsampling the majority group, which empirically achieves low minority error in the overparameterized regime, even though the standard approach of upweighting the minority fails.
Overall, our results suggest a tension between using overparameterized models versus using all the training data for achieving low worst-group error.

\end{abstract}

\section{Introduction}\label{sec:intro}

\begin{figure}[!t]
  \centering
  \includegraphics[width=0.4\textwidth]{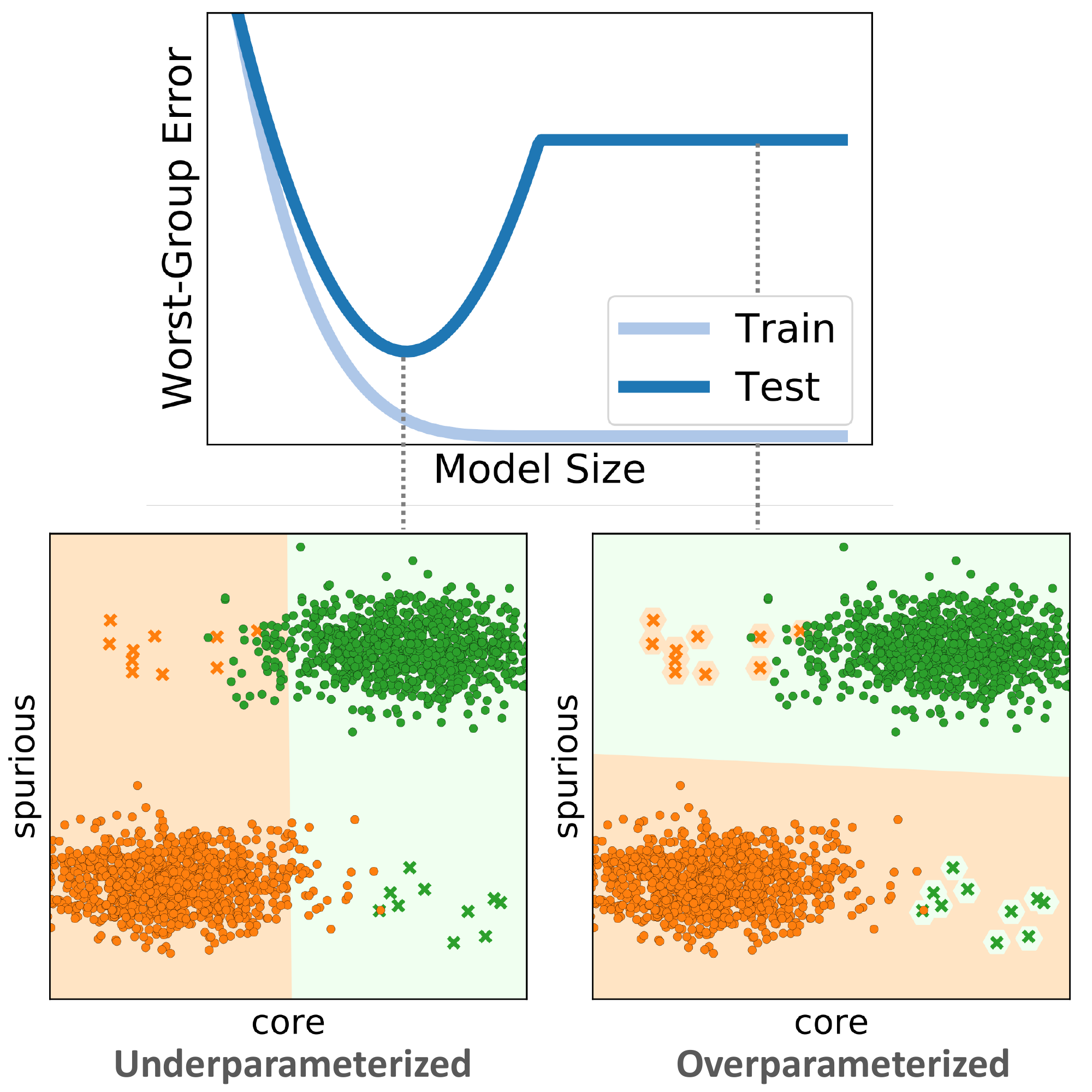}
  \vskip -3mm
  \caption{\textbf{Top}: Overparameterization \emph{hurts} test error on the worst group when models are trained with the reweighted objective that upweights minority groups (Equation~\ref{eqn:reweight}). Without reweighting, models have poor worst-group error regardless of model size (\refapp{appendix_erm}).
  \textbf{Bottom}: Consider data points $(x, y)$, where $x \in \R^2$ comprises a core feature $\xcau$ (x-axis) and a spurious feature $\xspu$ (y-axis). The label $y$ is highly correlated with $\xspu$, except on two minority groups (crosses).
  Underparameterized models use the core feature (left), but overparameterized models  use the spurious feature and memorize the minority points (right).
  }
  \vskip -3mm
  \label{fig:abstract}
\end{figure}

\begin{figure}[h]
  \centering
  \includegraphics[width=0.45\textwidth]{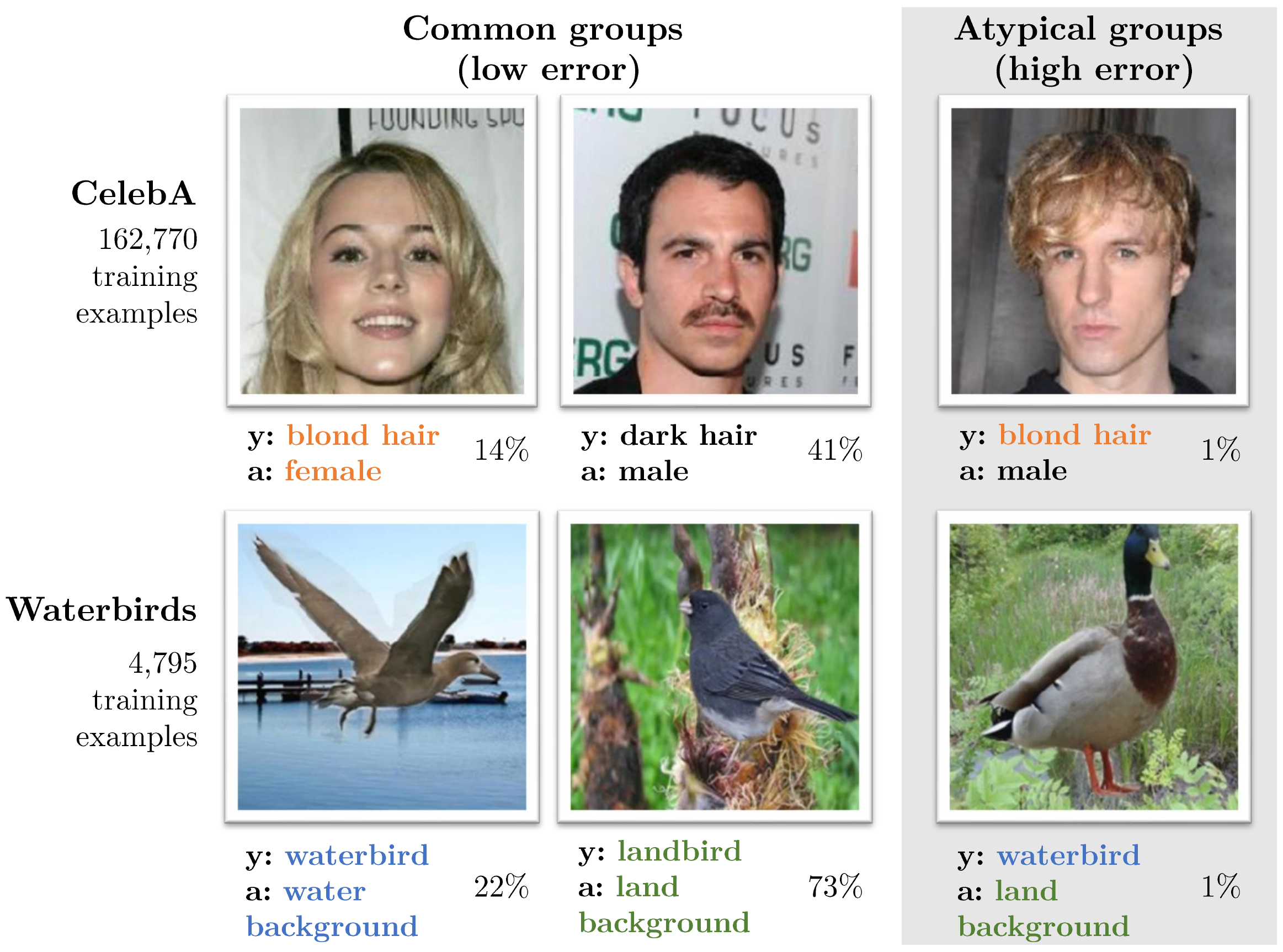}
  \vskip -3mm
  \caption{We consider two image datasets, CelebA and Waterbirds, where the label $y$ is correlated with a spurious attribute $a$
  in a majority of the training data. The \% beside each group shows its frequency in the training data.
  To measure how robust a model is to the spurious attribute, we divide the data into groups based on $(y, a)$ and record the highest error incurred by a group.
  Figure adapted from \citet{sagawa2020group}.}
  \vskip -3mm
  \label{fig:dataset}
\end{figure}

The typical goal in machine learning is to minimize the average error on a test set that is
independent and identically distributed (i.i.d.) to the training set.
A large body of prior work has shown that overparameterization---increasing model size beyond the point of zero training error---improves average test error in a variety of settings,
both empirically (with neural networks, e.g., \citet{nakkiran2019deep}) and theoretically (with linear and random projection models, e.g., \citet{belkin2019reconciling,mei2019generalization}).

However, recent work has also demonstrated that models with low average error can still fail on particular groups of data points \citep{blodgett2016, hashimoto2018repeated, buolamwini2018gender}.
This problem of high worst-group error arises especially in the presence of spurious correlations, such as strong associations between label and background in image classification \citep{mccoy2019right,sagawa2020group}.
To mitigate this problem, common approaches reduce the worst-group training loss, e.g., through distributionally robust optimization (DRO) or simply upweighting the minority groups.
\citet{sagawa2020group} showed these approaches improve worst-group error on strongly regularized neural networks but fail to help standard neural networks that can achieve zero training error, suggesting that increasing model capacity by reducing regularization---and perhaps by increasing overparameterization as well---can exacerbate spurious correlations.

In this paper, we investigate why overparameterization exacerbates spurious correlations
under the above approach of upweighting minority groups.
We first confirm on two image datasets (\reffig{dataset}) that directly increasing overparameterization (i.e., increasing model size) indeed hurts worst-group error, leading to models that are highly inaccurate on the minority groups where the spurious correlation does not hold (\refsec{results_empirical}).
In contrast, their underparameterized counterparts obtain much better worst-group error, but do worse on average.
We also confirm that models trained via empirical risk minimization (i.e., without upweighting the minority) have poor worst-group test error regardless of whether they are under- or overparameterized.
Through simulations on a synthetic setting, we further identify two properties of the training data that modulate the effect of overparameterization:
(i) the relative sizes of the majority versus minority groups,
and (ii) how informative the spurious features are relative to the core features (\refsec{results_props}).

% We then study when overparameterization hurts the worst-group error through simulations on a synthetic setting where the input $x$ is partitioned into core features that encode the actual label $y$ and spurious features that encode the spurious attribute $a$ (\refsec{results_props}).
% We identify

Why does overparameterization exacerbate spurious correlations?
Underparameterized models do not rely on spurious features because that would incur high training error on the (upweighted) minority groups where the spurious correlation does not hold.
In contrast, overparameterized models can always obtain zero training error by memorizing training examples, and instead rely on their inductive bias to pick a solution---which features to use and which examples to memorize---out of all solutions with zero training error.
Our results suggest an intuitive story of why overparameterization can hurt: because overparameterized models can have an inductive bias towards ``memorizing'' fewer examples
 (\reffig{abstract}).
If (i) the majority groups are sufficiently large and (ii) the spurious features are more informative than the core features for these groups, then overparameterized models could choose to use the spurious features because it entails less memorization, and therefore suffer high worst-group test error.
We test this intuition through simulations and formalize it in a theoretical analysis (\refsec{results_mechanism}).

Our analysis also leads to the counterintuitive result that on overparameterized models, subsampling the majority groups is much more effective at improving worst-group error than
upweighting the minority groups.
Indeed, an overparameterized model trained on a subset of $<$5\% of the data performs similarly (on average and on the worst group) to an underparameterized model trained on all the data (Section~\ref{sec:results_subsampling}).
This suggests a possible tension between using overparameterized models and using all the data;
average error benefits from both, but improving worst-group error seems to rely on using only one but not both.

\section{Setup}
\label{sec:setup}

\textbf{Spurious correlation setup.}
We adopt the setting studied in \citet{sagawa2020group},
where each example comprises the input features $x$, a label (core attribute) $y \in \sY$, and a spurious attribute $a \in \sA$. Each example belongs to a group $g\in\sG=\sY \times \sA$, where $g=(y,a)$.
Importantly, the spurious attribute $a$ is correlated with the label $y$ in the training set.
We focus on the binary setting in which $\sY = \{1,-1\}$ and $\sA = \{1,-1\}$.

\textbf{Applications.}
We study two image classification tasks (\reffig{dataset}).
In the first task, the label is spuriously correlated with demographics:
specifically, we use the CelebA dataset \citep{liu2015deep} to classify hair color between the labels $\sY = \{\text{blonde, non-blonde}\}$,
which are correlated with the gender $\sA = \{\text{female, male}\}$.
In the second task, the label is spuriously correlated with image background.
We use the Waterbirds dataset (based on datasets from \citet{wah2011cub, zhou2017places} and modified by \citet{sagawa2020group}) to classify between the labels $\sY = \{\text{waterbird, landbird}\}$, which are spuriously correlated with the image background $\sA =  \{\text{water background, land background}\}$.
See \refapp{appendix_details} for more dataset details.

\textbf{Objectives and metrics.}
We evaluate a model $w$ by its \emph{worst-group} error,
\begin{align}
  \label{eqn:roberr}
  \roberr{w} := \max_{g\in\sG}{\E_{x,y \mid g}}\left[\ell_{\mathsf{0-1}}(w; (x,y))\right],
\end{align}
where $\ell_{\mathsf{0-1}}$ is the 0-1 loss. In other words, we measure the error (\% of examples that are incorrectly labeled) in each group, and then record the highest error across all groups.
The standard approach to training models is empirical risk minimization (ERM):
given a loss function $\ell$, find the model $w$ that minimizes the average training loss
\begin{align}
\label{eqn:erm}
\hat{\mathcal{R}}_\mathsf{ERM}(w) = \hat{\E}_{(x,y,g)}\left[\ell(w; (x,y))\right].
\end{align}
However, in line with \citet{sagawa2020group},
we find that models trained via ERM have poor worst-group test error regardless of whether they are under- or overparameterized (\refapp{appendix_erm}).
To achieve low worst-group test error, prior work proposed modified objectives that focus on the worst-group loss,
such as group distributionally robust optimization (group DRO) which directly optimizes for the worst-group training loss \citep{hu2018does, sagawa2020group} or reweighting \citep{shimodaira2000improving,byrd2019effect}.
\citet{sagawa2020group} showed that both approaches can help worst-group loss,
though group DRO is typically more effective.
For simplicity, we focus on the well-studied reweighting approach, which optimizes
\begin{align}
\label{eqn:reweight}
\hat{\mathcal{R}}_\mathsf{reweight}(w) = \hat{\E}_{(x,y,g)}\left[\frac{1}{\hat{p}_g}\ell(w; (x,y))\right],
\end{align}
where $\hat{p}_g$ is the fraction of training examples in group $g$.
The intuition behind reweighting is that it makes each group contribute the same weight to the training objective: that is, minority groups are upweighted, while majority groups are downweighted.
Note that this approach requires the groups $g$ to be specified at training time, though not at test time.

\section{Overparameterization hurts worst-group error}
\label{sec:results_empirical}
\citet{sagawa2020group} observed that decreasing $L_2$ regularization hurts worst-group error.
Though increasing overparameterization and reducing regularization can have different effects \citep{zhang2017understanding, mei2019generalization}, this suggests that overparameterization might similarly exacerbate spurious correlations.
Here, we show that directly increasing overparameterization (model size) indeed hurts worst-group error even though it improves average error.

\textbf{Models.}
We study the CelebA and Waterbirds datasets described above.
For CelebA, we train a ResNet10 model \citep{he2016resnet},
varying model size by increasing the network width from
1 to 96, as in \citet{nakkiran2019deep}.
For Waterbirds, we use logistic regression over random projections, as in \citet{mei2019generalization}.
Specifically, let $x \in \R^d$ denote the input features, which we obtain by passing the input image through a pre-trained, fixed ResNet-18 model.
We train an unregularized logistic regression model over the feature representation $\relu(W x) \in \R^m$, where $W \in \R^{m \times d}$ is a random matrix with each row sampled uniformly from the unit sphere $\mathbb{S}^{d-1}$.
We vary model size by increasing the number of projections $m$ from 1 to 10,000.
We train each model by minimizing the reweighted objective (Equation~\refeqn{reweight}).
For more details, see \refapp{appendix_details}.

\textbf{Results.}
Overparameterization improves average test error across both datasets, in line with prior work \citep{belkin2019reconciling,nakkiran2019deep} (\reffig{image_curve}).
However, in stark contrast, overparameterization \emph{hurts} worst-group error:
the best worst-group test error is achieved by an \emph{underparameterized} model with non-zero training error.
On CelebA, the smallest model (width 1) has 12.4\% worst-group training error but comparatively low worst-group test error of 25.6\%. As width increases, training error goes to zero but worst-group test error gets worse, reaching $>$60\% for overparameterized models with zero training error.
Similarly, on Waterbirds,
an underparameterized model with $90$ random features and worst-group training error of 17.7\% obtains the best worst-group test error of 26.6\%,
while overparameterized models with zero training error yield worst-group test error of 42.4\% at best.

In \refapp{appendix_reg}, we also confirm that stronger regularization improves worst-group error but hurts average error in overparameterized models, while it has little effect on both worst-group and average error in underparameterized models.
However, we focus on understanding the effect of overparameterization in the remainder of the paper.

\textbf{Discussion.}
Why does overparameterization hurt worst-group test error?
We make two observations.
First, in the overparameterized regime, the smallest groups incur the highest test error (blonde males in CelebA and waterbirds on land background in Waterbirds), despite having zero training error.
In other words, overparameterized models perfectly fit the minority points at training time,
but seem to do so by using patterns that do not generalize.
We informally refer to this behavior as ``memorizing'' the minority points.

Second, underparameterized models do obtain low worst-group error by learning patterns that generalize to both majority and minority groups.
Therefore, overparameterized models should also be able to learn these patterns while attaining zero training error (e.g., by memorizing the training points that the underparameterized model cannot fit).
Despite this, overparameterized models seem to learn patterns that generalize well on the majority but do not work on the minority (such as the spurious attributes $a$ in \reffig{dataset}).

What makes overparameterized models memorize the minority instead of learning patterns that generalize well on both majority and minority groups?
We study this question in the next two sections:
in \refsec{results_props}, we use simulations to understand properties of the data distribution that give rise to this trend,
and in \refsec{results_mechanism} we analyze a simplified linear setting and show how the inductive bias of models towards memorizing fewer points can lead to overparameterized models choosing to use spurious correlations.

\begin{figure}[!t]
\centering
\includegraphics[width=0.5\textwidth]{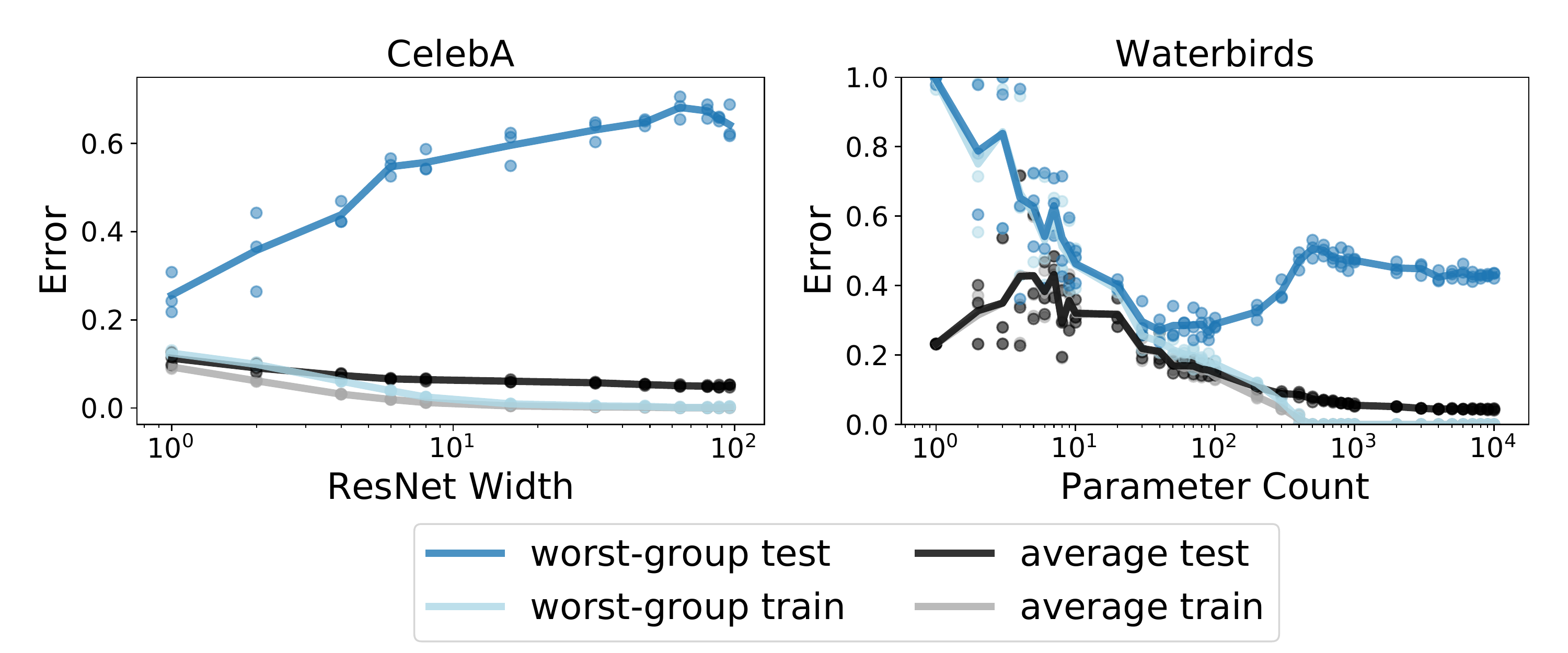}
  \vskip -0.1in
  \caption{Increasing overparameterization (i.e., increasing model size) hurts the worst-group test error even though it improves the average test error. Here, we show results for models trained on the reweighted objective for CelebA (left) and Waterbirds (right).
  }
  \label{fig:image_curve}
\vskip -0.2in
\end{figure}

\section{Simulation studies}\label{sec:results_props}
The discussion in \refsec{results_empirical} suggests two properties of the training distribution that modulate the effect of overparameterization on worst-group error. Intuitively, overparameterized models should be more incentivized to use the spurious features and memorize the minority groups if (i) the proportion of the majority group, $\pmajmath$, is higher, and (ii) the ratio of how informative the spurious features are relative to the core features, $\scrmath$, is higher.
In this section, we use simulations to confirm these intuitions and probe how $\pmajmath$ and $\scrmath$ affect worst-group error in overparameterized models.

\subsection{Synthetic experiment setup}\label{sec:toy_setup}

\textbf{Data distribution.}
We construct a synthetic dataset that replicates the empirical trends in \refsec{results_empirical}.
As in \refsec{setup}, the label $y \in \{1,-1\}$ is spuriously correlated with a spurious attribute $a \in \{1,-1\}$.
We divide our training data into four groups accordingly: two majority groups with $a=y$, each of size $\nmaj/2$, and two minority groups with $a=-y$, each of size $\nmin/2$.
We define $\ntot = \nmaj + \nmin$ as the total number of training points, and $\pmajmath = \nmaj/\ntot$ as the fraction of majority examples.
The higher $\pmajmath$ is, the more strongly $a$ is correlated with $y$ in the training data.

Each $(y, a)$ group has its own distribution over input features $x = [\xcau, \xspu] \in \R^{2d}$ comprising core features $\xcau \in \R^d$  generated from the label/core attribute $y$, and spurious features $\xspu \in \R^d$ generated from the spurious attribute $a$:
\begin{align}
  \label{eqn:toy1}
  \xcau \mid y \sim \sN(y\bold{1},\sigcau^2I_d)\nonumber\\
  \xspu \mid a \sim \sN(a\bold{1},\sigspu^2I_d).
\end{align}
The core and spurious features are both noisy and encode their respective attributes at different signal-to-noise ratios. We define the \emph{spurious-core information ratio} (SCR) as $\scrmath = \sigcau^2/\sigspu^2$. The higher the SCR, the more signal there is about the spurious attribute in the spurious features, relative to the signal about the label in the core features.

Compared to the image datasets we studied in \refsec{results_empirical}, this synthetic dataset offers two key simplifications.
First, the only differences between groups stem from their differences in $(y, a)$, which isolates the effect of flipping the spurious attribute $a$. In contrast, in real datasets, groups can differ in other ways, e.g., more label noise in one group.
Second, the relative difficulty of estimating $y$ versus $a$ is completely governed by  changing $\sigcau^2$ and $\sigspu^2$.
In contrast, real datasets have additional complications, e.g., estimating $y$ might involve a more complex function of the input $x$ than estimating $a$, and there might be an inductive bias towards learning a simpler model over a more complex one.

In all of the experiments below, we fix the total number of training points $n$ to $3000$, and set $d=100$ (so each input $x$ has $2d=200$ dimensions).
Unless otherwise specified,
we set the majority fraction $\pmajmath = 0.9$ and the noise levels $\sigspu^2 = 1$ and $\sigcau^2 = 100$ to encourage the model to use the spurious features over the core features.

\textbf{Model.} To avoid the complexities of optimizing neural networks, we follow the same random features setup we used for Waterbirds in \refsec{results_empirical}: unregularized logistic regression using the reweighted objective
on the random feature representation $\relu(W x) \in R^m$, where $W \in \R^{m \times d}$ is a random matrix \citep{mei2019generalization}.

\subsection{Observations on synthetic dataset}
\label{sec:synthetic_observations}

\textbf{The synthetic dataset replicates the trends we observe on real datasets.}
\reffig{toy_curve_base} shows how average and worst-group error change with the number of parameters/random projections $m$. This matches the trends we obtained on  CelebA and Waterbirds in \refsec{results_empirical}.
The best worst-group test error of 28.5\% is achieved by an underparameterized model, whereas highly overparameterized models achieve high worst-group test error that plateaus at around 55\%. In contrast, the average test error is better for overparameterized models than for underparameterized models.

\begin{figure}[!t]
\centering
\includegraphics[width=0.4\textwidth]{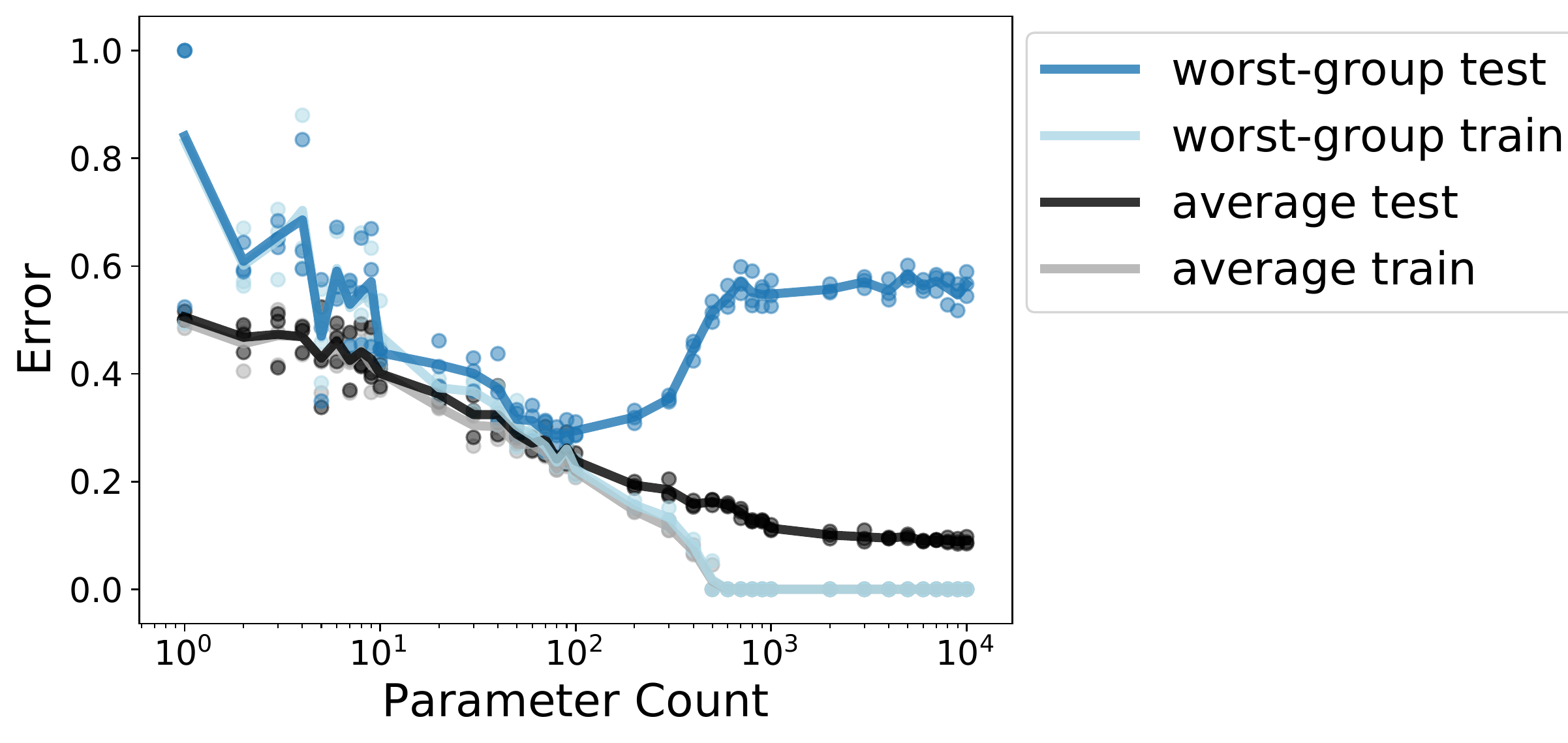}
\vskip -0.1in
  \caption{Overparameterization hurts worst-group test error but improves average test error on synthetic data, reproducing the trends we observe in real data.}
  \label{fig:toy_curve_base}
\vskip -0.1in
\end{figure}

\begin{figure}[!t]
\centering
\includegraphics[width=0.48\textwidth]{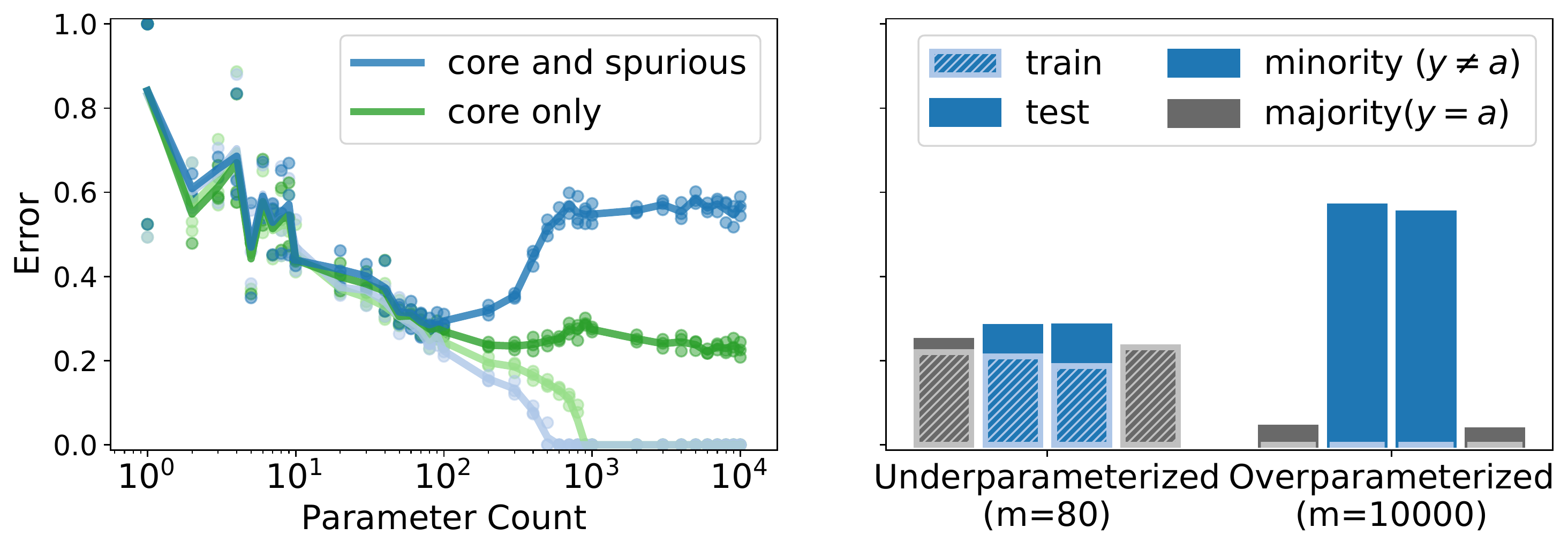}
\vskip -0.1in
  \caption{Overparameterized models have poor worst-group performance on the synthetic data because they rely on spurious features. \textbf{Left}: removing the spurious feature (green) eliminates the detrimental effect of overparameterization. \textbf{Right}: overparamerized models do well on the majority groups where the spurious features match the label, but poorly on the minority groups.}
  \label{fig:toy_intuition}
\vskip -0.1in
\end{figure}

\textbf{Overparameterized models use spurious features.}
\reffig{toy_intuition}-Right shows that overparameterized models have high test error on minority groups ($a=-y$) despite zero training error, but perform very well on the majority groups ($a=y$).
Since the only difference between the minority and majority groups in the synthetic dataset is the relative signs of the core and spurious attributes, this suggests overparameterized models are using spurious features and simply memorizing the minority groups to get zero training error, consistent with our discussion in \refsec{results_empirical}.
In contrast, the underparameterized model has low training and test errors across all groups,
suggesting that it relies mainly on core features.

These results imply that the degradation in the worst-group test error is due to the spurious features.
We confirm that overparameterization no longer hurts when we ``remove'' the spurious features by replacing them with noise centered around zero (i.e., we replace the mean of $\xspu$ by 0).
In this case, the best worst-group test error is now obtained by an overparameterized model, as shown in \reffig{toy_intuition}-Left.

\subsection{Distributional properties}
What properties of the training data make overparameterization hurt worst-group error?
We study (i) $\pmajmath$, which controls the relative size of majority to minority groups, and (ii) $\scrmath$, the relative informativeness of spurious to core features.
In the synthetic dataset, overparameterization hurts worst-group test error only when both are sufficiently high. In contrast, overparameterization helps average test error regardless; see \refapp{appendix_avg}.

\textbf{Effect of the majority fraction $\pmajmath$.}
We observe that increasing $\pmajmath = \nmaj / n$, which controls the relative size of the majority versus minority groups, makes overparameterization hurt worst-group error more (\reffig{toy_knob}).
When the groups are perfectly balanced with $\pmajmath=0.5$, overparameterization no longer hurts the worst-group test error, with overparameterized models achieving better worst-group test error than all underparameterized models.
This suggests that group imbalance can be a key factor inducing the detrimental effect of overparameterization.

\textbf{Effect of the spurious-core information ratio $\scrmath$.}
Next, we characterize the effect of $\scrmath  = \sigcau^2/\sigspu^2$, which measures the relative informativeness of the spurious versus core features. A high $\scrmath$ means that the spurious features are more informative.
We vary $\scrmath$ by changing $\sigspu^2$ while keeping $\sigcau^2 = 100$ fixed, since this does not change the best possible worst-group test error (with a model that uses only the core features $\xcau$).
\reffig{toy_knob} shows that the higher $\scrmath$ is, the more overparameterization hurts.
As $\scrmath$ increases, the spurious features become more informative, and overparameterized models rely more on them than the core features;
underparameterized models outperform overparameterized models only for sufficiently large $\scrmath \ge 1$.
Note that increasing $\scrmath$ does not significantly affect the worst-group test error in the underparameterized regime, since the core features $\xcau$ are unaffected. In contrast, increasing the majority fraction $\pmajmath$ hurts the worst-group test error in both underparameterized and overparameterized models.

\begin{figure}[!t]
\centering
\includegraphics[width=0.485\textwidth]{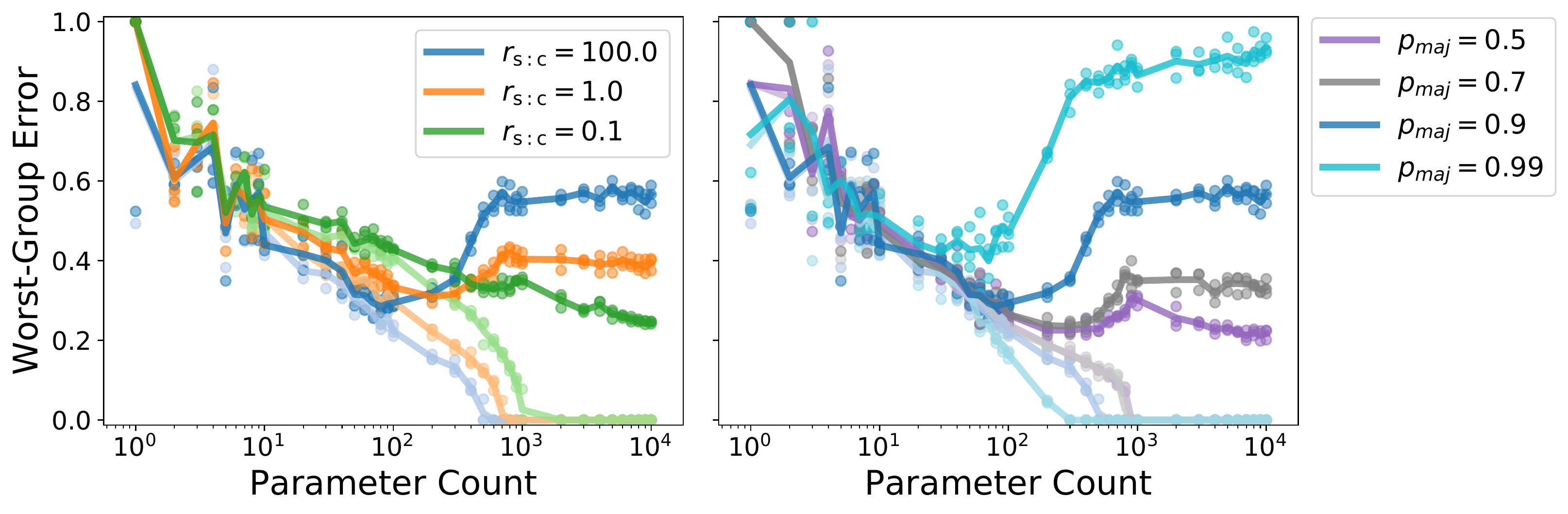}
\vskip -0.1in
  \caption{The higher the majority fraction $\pmajmath$ and the spurious-core information ratio $\scrmath$, the more overparameterization hurts the worst-group test error. With sufficiently low $\pmajmath$ and $\scrmath$, overparameterization switches to helping worst-group test error.}
  \label{fig:toy_knob}
\vskip -0in
\end{figure}

\subsection{An intuitive story}
We return to the question of what makes overparameterized models memorize the minority instead of learning patterns that generalize on both majority and minority groups. The simulation results above show that of all overparameterized models that achieve zero training error, the inductive bias of the model class and training algorithm favors models that use spurious features which generalize only for the majority groups, instead of learning to use core features that also generalize well on the minority groups.

What is the nature of this inductive bias?
Consider a model that predicts the label $y$ by returning its estimate of the spurious attribute $a$ from $\xspu$, taking advantage of the fact that $y$ and $a$ are correlated in the training data.
To get achieve zero training error, it will need to memorize the points in the minority group, e.g., by exploiting variations due to noise in the features $x$.
On the other hand, consider a model that predicts $y$ by returning a direct estimate of $y$ based on the core features $\xcau$.
Because $\xcau$ provides a noisier estimate of $y$ than $\xspu$ does for $a$, this model will need to memorize all points for which $\xcau$ gives an inaccurate prediction of $y$ due to noise.
Since the estimators of the core and spurious attributes are equally easy to learn,
the main difference between these two models is the number of examples to be memorized.

We therefore hypothesize that \emph{the inductive bias favors memorizing as few points as possible}.
This is consistent with the results above:
the model uses $\xspu$ and memorizes the minority points only when the fraction of minority points is small (high majority fraction $\pmajmath$).
Similarly, the model uses $\xspu$ over $\xcau$ to fit the majority points only when the spurious features are less noisy (high $\scrmath$) and therefore require less memorization to obtain zero training error than the core features.
In the next section, we make this intuition formal by analyzing a related but simpler linear setting.

\section{Theoretical analysis}\label{sec:results_mechanism}
In this section, we show how the inductive bias against memorization leads to overparameterization exacerbating spurious correlations.
Our analysis explicates the effect of the inductive bias and the importance of the data parameters $\pmajmath$ and $\scrmath$ discussed in \refsec{results_props}.

The synthetic setting discussed in \refsec{results_props} is difficult to analyze because of the non-linear random projections, so we introduce a linear \emph{explicit-memorization} setting that allows us to precisely define the concept of memorization. For clarity, we refer to the previous synthetic setting in \refsec{results_props} as the \emph{implicit-memorization} setting.
In \refapp{memorization_comparison}, we show empirically that models in these two settings behave similarly in the overparameterized regime, though they differ in the underparameterized regime.

In the previous implicit-memorization setting, we varied model size and memorization capacity by varying the number of random projections of the input. In the new explicit-memorization setting, we instead use linear models that act directly on the input and introduce explicit ``noise features'' that can be used to memorize. We vary the memorization capacity by varying the number of explicit noise features.

\subsection{Explicit-memorization setup}\label{sec:implicit_setup}
\textbf{Training data.}
We consider input features $x = [\xcau, \xspu, \xnoise]$, where the core feature $\xcau \in \R$ and the spurious feature $\xspu \in \R$ are scalars. As in the implicit-memorization setup, they are generated based on the label and the spurious attribute, respectively:
\begin{align*}
  \xcau \mid y \sim \sN(y,\sigcau^2), ~~&~~ \xspu \mid a \sim \sN(a,\sigspu^2).
\end{align*}
The ``noise'' features $\xnoise \in \R^\nnoise$ are generated as
\begin{align*}
  \xnoise &\sim \sN\left(0, \frac{\signoise^2}{N} I_N\right),
\end{align*}
where $\signoise^2$ is a constant. The scaling by $1/N$ ensures that for large $\nnoise$, the norm of the noise vectors $\| \xnoise \|_2^2 \approx \signoise^2$ is approximately constant with high probability.
Intuitively, when $N$ is large, overparameterized models can use $\xnoise$ to fit a training point $x$ without affecting its predictions on other points, thereby memorizing $x$. We formalize this notion of memorization later in \refsec{proof_outline}.

As before, the training data is composed of four groups, each corresponding to a combination of the label $y\in \{-1,1\}$ and the spurious attribute $a\in \{-1, 1 \}$: two majority groups with $a=y$, each of size $\nmaj/2$, and two minority groups with $a=-y$, each of size $\nmin/2$.
Combined, there are $n$ training examples $\{(x^{(i)}, y^{(i)})\}_{i=1}^n$.

\textbf{Model.}
We study unregularized logistic regression on the input features $x \in \R^{N+2}$.
As before, we consider the reweighted estimator $\estrw$.
When the training data is linearly separable, the minimizer of the unregularized logistic loss on the training data is not well-defined.
We therefore define $\estrw$ in terms of the sequence of $L_2$-regularized models $\estrwl$:
\begin{align}
  \estrwl \eqdef \underset{w\in\R^{N+2}}{\arg\min} ~~ \hat{\E}_{(x,y,g)}\biggl[\frac{1}{\hat{p}_g}\ell(w; (x,y))\biggr] + \frac{\lambda}{2} \|w\|_2^2,\nonumber
\end{align}
where $\ell$ is the logistic loss and $\hat{p}_g$ is the fraction of training examples in group $g$.
Since scaling a model does not affect its 0-1 error,
we define $\estrw$ as the limit of this sequence, scaled to unit norm, as the regularization strength $\lambda \to 0^+$:
\begin{align}
  \label{eqn:estrw}
  \estrw \eqdef \lim_{\lambda \to 0^+} \frac{\estrwl}{\|\estrwl\|_2}.
\end{align}
In the underparameterized regime, the training data is not linearly separable and
we simply have $\estrw = \estrw_0 / \|\estrw_0\|_2$.
In the overparameterized regime where $\nnoise\gg n$, the training data is linearly separable,
and \citet{rosset2004margin} showed that $\estrw = \estmmprimal$, where $\estmmprimal$ is the max-margin classifier
\begin{align}
  \label{eqn:estmmprimal}
  \estmmprimal \eqdef \underset{\|w\|_2 = 1}{\arg\max} ~ \min_i y^{(i)}(w \cdot x^{(i)}).
\end{align}
The equivalence $\estrw = \estmmprimal$ holds regardless of the reweighting by $1/\hat{p}_g$: if we define the ERM estimator $\esterm$ analogously to \refeqn{estrw} without the reweighting, it is also equal to $\estmmprimal$.
We will therefore analyze $\estmmprimal$ in the overparameterized regime since it subsumes both $\estrw$ and $\esterm$.

We also note that if we use gradient descent to directly optimize the unregularized logistic regression objective (either reweighted or not), the resulting solution after scaling to unit norm also converges to $\estmmprimal$ as the number of gradient steps goes to infinity \citep{soudry2018implicit}.

\subsection{Analysis of worst-group error}\label{sec:proof_outline}
We now state our main analytical result: in the explicit-memorization setting, the worst-group test error of a sufficiently overparameterized model is greater than $1/2$ (worse than random) under certain settings of $\sigspu^2, \sigcau^2, \nmaj, \nmin$.
In contrast, underparameterized models attain reasonable worst-group error even under such a setting.
\begin{restatable}[]{thm}{main}
  \label{thm:main}
  For any $\pmajmath \geq \bigl(1-\frac{1}{2001}\bigr)$, $\sigcau^2 \geq 1$, $\sigspu^2 \leq \frac{1}{16 \log 100 \nmaj}$, $\signoise^2 \leq \frac{\nmaj}{600^2}$ and $\nmin \geq 100$,
  there exists $\nnot$ such that for all $\nnoise > \nnot$ (overparameterized regime), with high probability over draws of the data,
  \begin{align}
    \roberr{\estmmprimal} \geq 2/3,
  \end{align}
  where $\estmmprimal$ is the max-margin classifier.

  However, for $N=0$ (underparameterized regime), with $\pmajmath=\bigl(1-\frac{1}{2001}\bigr)$, $\sigcau^2 = 1$, and $\sigspu^2 = 0$, and in the asymptotic regime with $\nmaj, \nmin \rightarrow \infty$, we have
  \begin{align}
    \roberr{\estrw} < 1/4,
  \end{align}
  where $\estrw$ minimizes the reweighted logistic loss.
 \end{restatable}
 The result in the overparameterized regime applies to the max-margin classifier $\estmmprimal$, which as discussed above subsumes both $\estrw$ and $\esterm$ when the data is linearly separable.
 The proof of Theorem~\ref{thm:main} appears in Appendix~\ref{sec:app-analysis}.

The conditions on $\sigspu^2$ and $\sigcau^2$ in Theorem~\ref{thm:main} above imply high spurious-core information ratio $\scrmath$.
Theorem~\ref{thm:main} therefore provides a setting where high $\pmajmath$ and high $\scrmath$ provably make overparameterized models obtain high worst-group error,
matching the trends we observed upon varying $\pmajmath$ and $\scrmath$ in the implicit-memorization setting (\reffig{toy_knob}).
Furthermore, underparameterized models obtain reasonable worst-group error despite these conditions, mirroring the observations in earlier sections.

\subsection{Overparameterization and memorization}
We now sketch the key ideas in the proof of Theorem~\ref{thm:main} (full proof in Appendix~\ref{sec:app-analysis}), focusing first on the overparameterized regime.
We start by establishing an inductive bias towards learning the minimum-norm model that fits the training data.
We then define memorization and show how the minimum-norm inductive bias translates into a bias against memorization.
Finally, we illustrate how the bias against memorization leads to learning the spurious feature and suffering high worst-group error.

\textbf{Minimum-norm inductive bias.}
Define a \emph{separator} as any model that correctly classifies all of the training points $(x,y)$ with margin $yw\cdot x \geq 1$. Then from standard duality arguments, $\estmmprimal$ can be rewritten as $\estmm/\| \estmm \|$, the scaled version of the \emph{minimum-norm separator} $\estmm$
\begin{align}
  \label{eqn:mm}
  \estmm \eqdef \underset{w\in\R^{N+2}}{\arg\min} \| w \|_2^2~~\text{s.t.}~~y^{(i)}(w \cdot x^{(i)}) \geq 1 ~\forall i.
\end{align}
Since scaling does not affect the 0-1 test error, it suffices to analyze $\estmm$. Equation \refeqn{mm} shows that out of the set of all separators (which all perfectly fit the training data), the inductive bias favors the separator with the minimum norm. We now discuss how this minimum-norm inductive bias favors less memorization.

\textbf{Memorization.}
For convenience, we denote the three components of a model $w$ as
\begin{align}
  w = \left[\cau{w}, \spu{w}, \noise{w}\right],
\end{align}
where $\cau{w} \in \R$, $\spu{w} \in \R$, and $\noise{w} \in \R^\nnoise$.
By the representer theorem, we can decompose $\noise{w}$ as follows:
\begin{align}
  \label{eqn:repr}
  \noise{w} = \sum_i \isw{w}{i} \ixnoi{i}.
\end{align}
In the overparameterized regime when $\nnoise \gg n$, a model can ``memorize'' a training point $\ix{i}$ via $\noise{w}$, in particular by putting a large weight $\isw{w}{i}$ in the direction of $\ix{i}$ (Equation~\refeqn{repr}):
\begin{definition}[$\gamma$-memorization]
  \label{def:mem}
  A model $w$ memorizes a point $\ix{i}$ if $|\isw{w}{i}| \geq \gamma^2/\signoise^2$ for some constant $\gamma \in \R$.
\end{definition}
Because the noise vectors of the training points (high-dimensional Gaussians) are nearly orthogonal for large $\nnoise$, the component $\isw{w}{i}\ixnoi{i}$ affects the prediction on $\ix{i}$, but not on any other training or test points.

This ability to memorize plays a crucial role in making overparameterized models obtain high worst-group error.
Intuitively, the minimum-norm inductive bias favors less memorization in overparameterized models. Roughly speaking, models that memorize more have larger weights $| \isw{w}{i} |$ on the noise vectors $\ixnoi{i}$. Since these noise vectors are nearly orthogonal and have similar norm, this translates into a larger norm  $\| \noise{w} \|_2^2$.

\textbf{Comparing using $\xcau$ versus using $\xspu$.} To illustrate how the inductive bias against memorization leads to high worst-group error, we consider two extreme sets of separators: (i) ones that use the spurious feature but not the core feature, denoted by $\Memmin$ (ii) ones that use the core feature but not the spurious feature, denoted by $\Memmaj$.
\begin{align}
  \label{eq:sep}
  \Memmin &\eqdef \{ w \in \R^{N+2}: w~\text{is a separator}, \cau{w} = 0\} \nonumber\\
  \Memmaj &\eqdef \{ w \in \R^{N+2}: w~\text{is a separator}, \spu{w} = 0\}.
\end{align}
In scenario (i), using the spurious feature $\xspu$ alone allows models to fit the majority groups very well.
Thus, models that use $\xspu$ only need to memorize the minority points.
In \refprop{spu-cost}, we construct a separator $\memmin \in \Memmin$ and show that its norm \emph{only} scales with the number of minority points $\nmin$.

Conversely, in scenario (ii), using the core feature $\xcau$ alone allows models to fit all groups equally well. However, when $\scrmath$ is high, $\xcau$ is noisier than $\xspu$, so models that use $\xcau$ still need to memorize a constant fraction of \emph{all} the training points. In \refprop{core-cost}, we show that norms of all separators $\memmaj \in \Memmaj$ are lower bounded by a quantity linear in the total number of training points $n$.

When the majority fraction $\pmajmath$ is sufficiently large such that $\nmin \ll n$, the separator $\memmin$ that uses $\xspu$ will have a lower norm than any separator $\memmaj \in \Memmaj$ that uses $\xcau$. Since the inductive bias favors the minimum-norm separator, it prefers a separator $\memmin$ that memorizes the minority points and suffers high worst-group error over any $\memmaj \in \Memmaj$.

\begin{restatable}[Norm of models using the spurious feature]{prop}{spucost}
  \label{prop:spu-cost}
  When $\sigcau^2, \sigspu^2$ satisfy the conditions in Theorem~\ref{thm:main}, there exists $N_0$ such that for all $N > N_0$, with high probability, there exists a separator $\memmin \in \Memmin$ such that
  \begin{align*}
    \| \memmin \|_2^2 &\leq \gamma_1^2 + \Bigg(\frac{\gamma_2 \nmin}{\signoise^2}\Bigg),
  \end{align*}
  for some constants $\gamma_1, \gamma_2 > 0$.
\end{restatable}
\begin{proof}[Proof sketch]
  To simplify exposition in this sketch, suppose that the noise vectors $\ixnoi{i}$ are orthogonal and have constant norm $\| \ixnoi{i} \|_2^2 = \signoise^2$.
  We construct a separator $\memmin \in \Memmin$ that does not use the core feature $\xcau$ as follows.
  Set $\spu{\memmin} = \gamma_1$ for some large enough constant $\gamma_1 > 0$.
  This is sufficient to satisfy the margin condition on the majority points:
  since $\sigspu^2$ is very small, w.h.p. all majority training points satisfy $\iy{i}(\ixspu{i} \gamma_1) \geq 1$.

  However, for the minority training points, the spurious attribute $a$ does not match the label $y$, and in order to satisfy the margin condition with a positive $\spu{\memmin}$, these $\nmin$ minority points have to be memorized.
  Since $\sigspu^2$ is very small, the decrease in the margin due to $\spu{\memmin} = \gamma_1$ is at most $- \rho \gamma_1$ w.h.p. for some constant $\rho$ that depends on $\sigspu^2$. To satisfy the margin condition, it thus suffices to set $\isw{w}{i}_\textmin=\iy{i}(1 + \rho \gamma_1) /\signoise^2$, and the bound on the norm follows.
  The full proof appears in Section~\ref{sec:app-spu-cost}.
\end{proof}

\begin{restatable}[Norm of models using the core feature]{prop}{corecost}
  \label{prop:core-cost}
  When $\sigcau^2, \sigspu^2$ satisfy the conditions in Theorem~\ref{thm:main} and $\nmin \geq 100$, there exists $N_0$ such that for all $\nnoise > N_0$, with high probability, all separators $\memmaj \in \Memmaj$ satisfy
  \begin{align*}
    \| \memmaj \|_2^2 \geq \frac{\gamma_3 n}{\signoise^2},
  \end{align*}
  for some constant $\gamma_3 > 0$.
\end{restatable}
\begin{proof}[Proof sketch]
 Any model $\memmaj \in \Memmaj$ has $\spu{\memmaj} = 0$ by definition. We show that a constant fraction of training points have to be $\gamma$-memorized in order to satisfy the margin condition. We do so by first showing that the probability that a training point $x$ satisfies the margin condition \emph{without} being $\gamma$-memorized cannot be too large.
 For simplicity, suppose again that the noise vectors $\ixnoi{i}$ are orthogonal and have constant norm $\| \ixnoi{i} \|_2^2 = \signoise^2$.
 Then this probability is $\BP \bigl(x_\mathsf{core} \cau{\memmaj} \leq 1 - \gamma^2 \bigr) \geq \Phi(-1/\sigcau)$ for small $\gamma$, where $\Phi$ is the Gaussian CDF.
 Hence, in expectation, at least a constant fraction of points from the training distribution need to be memorized in order for $\memmaj$ to satisfy the margin condition. With high probability, this is also true on the training set consisting of $n$ points (via the DKW inequality) and the bound on the norm follows. The full proof appears in Section~\ref{sec:app-core-cost}.
\end{proof}
In the full proof of Theorem~\ref{thm:main} in Appendix~\ref{sec:app-analysis}, we generalize the above ideas to consider all separators in $\R^{N+2}$ instead of just the separators in $\Memmin \bigcup \Memmaj$. Note the importance of both $\scrmath$ and $\pmajmath$: when $\scrmath$ is high, models that use $\xspu$ only need to memorize the minority groups (Proposition~\ref{prop:spu-cost}), and when $\pmajmath$ is also high, these models end up memorizing fewer points than models that use $\xcau$ and have to memorize a constant fraction of the entire training set (Proposition~\ref{prop:core-cost}).

\section{Subsampling}\label{sec:results_subsampling}
Our results above highlight the role of the majority fraction $\pmajmath$ in determining if overparameterization hurts worst-group test error.
When $\pmajmath$ is large, the inductive bias favors using spurious features because it entails memorizing only a relatively small number of minority points, while the alternative of using core features requires memorizing a large number of majority points.
This suggests that reducing the memorization cost of using core features by directly removing some majority points could induce overparameterized models to obtain low worst-group error.

Here, we show that this approach of \emph{subsampling} the majority group achieves good worst-group test error on the datasets studied above.
Subsampling creates a new \emph{group-balanced} dataset by randomly removing training points in all other groups to match the number of points from the smallest group
\citep{japkowicz2002class,haixiang2017learning,buda2018systematic}.
We then train a model to minimize the average loss on this subsampled dataset.
For a precise description, see \refapp{appendix_subsampling}.

\begin{figure}[b]
\centering
\includegraphics[width=0.48\textwidth]{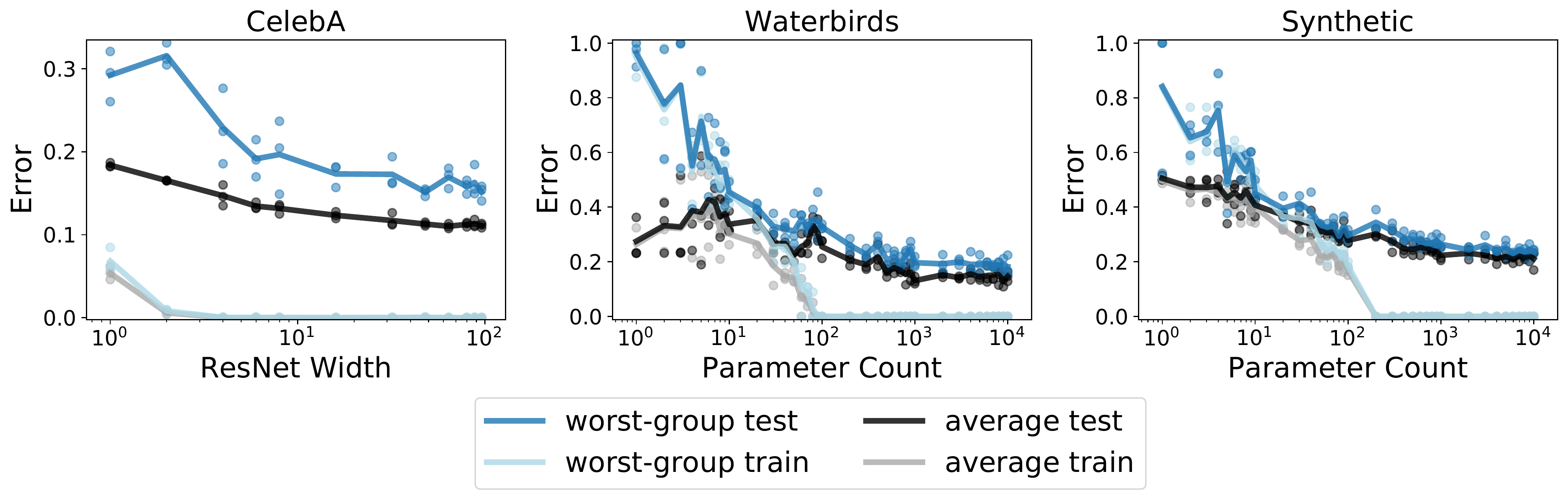}
\vskip -0.1in
  \caption{Overparameterization helps worst-group test error when training via subsampling, which involves creating a group-balanced dataset by reducing the number of majority points and minimizing average training loss on the new dataset.}
  \label{fig:subsample}
\end{figure}

\reffig{subsample} shows that overparameterized models trained via subsampling (Equation~\ref{eqn:subsample}) obtain low worst-group error on the CelebA, Waterbirds, and synthetic (implicit-memorization) datasets. Across all three datasets, training via subsampling makes increasing overparameterization help \emph{both} average and worst-group test error.
Moreover, overparameterized models trained on subsampled data are comparable to or better than the best models trained on the full dataset (i.e., underparameterized models trained with reweighting).

Subsampling seems wasteful since it throws away a large fraction of the training data: we only use 3.4\% of the full training data for CelebA, 4.6\% for Waterbirds, and 10\% for the synthetic dataset. However, the results above show that subsampling in overparameterized models matches or outperforms reweighting with underparameterized models. For example, on CelebA, an overparameterized model trained via subsampling obtains 11.1\% average test and 15.1\% worst-group test error, whereas an underparameterized model trained with reweighting obtains 11.3\% average and 25.6\% worst-group test error.

\textbf{Subsampling vs. reweighting.} Both subsampling and reweighting artificially balance the groups in the training data,
and previous work on imbalanced datasets has concluded that reweighting is typically at least as effective as subsampling \citep{buda2018systematic}.
However, we find a clear difference between subsampling and reweighting in the overparameterized regime: increasing overparameterization with reweighting increases worst-group error, while doing so with subsampling decreases worst-group error.
The intuition developed in Sections~\ref{sec:results_props} and~\ref{sec:results_mechanism} shed some light on this difference.
Consider an overparameterized model:
as in \refsec{implicit_setup}, reweighting does not change the learned model which is the max-margin classifier.
However, subsampling reduces $\pmajmath$. Recall that the inductive bias favors spurious features when the alternative of using core features requires memorizing a large number of training points.
By reducing $\pmajmath$, we reduce this memorization cost associated with core features, thereby inducing the model to use core features and achieve low worst-group test error.

\section{Related work}\label{sec:related}

\textbf{The effect of overparameterization.}
The effect of overparameterization on average test error has been widely studied.
In what is commonly referred to as ``double descent'', increasing model size beyond zero training error decreases test error, despite conventional wisdom that overfitting should increase test error.
This behavior has been observed empirically~\citep{belkin2019reconciling, opper1995statistical, advani2017high, nakkiran2019deep} and shown analytically in high-dimensional regression~\citep{hastie2019surprises, bartlett2019benign, mei2019generalization}.
These works focus on average test error and are consistent with our findings there.
However, our focus is on worst-group test error, particularly when the groups are defined based on spurious attributes, and in this paper we establish that worst-group test error can behave quite differently from average test error.

Increasing overparameterization can actually improve model robustness to some types of distributional shifts \citep{hendrycks2019natural, hendrycks2019benchmarking, yang2020rethinking}. In this light, our results show that the effect of overparameterization on model robustness can depend heavily on the dataset (e.g., properties like $\pmajmath$ and $\scrmath$), type of distributional shift, and training procedure.

\textbf{Worst-group error.}
Prior work on improving worst-group error focused on the underparameterized regime,
with methods based on
weighting/sampling \citep{shimodaira2000improving,japkowicz2002class,buda2018systematic,cui2019class},
distributionally robust optimization (DRO) \citep{bental2013robust,namkoong2017variance,oren2019drolm},
and fair algorithms \citep{dwork2012,hardt2016,kleinberg2017}.
Our focus is on the overparameterized, zero-training-error regime;
here, previous methods based on reweighting and DRO are ineffective \citep{wen2014robust,byrd2019effect,sagawa2020group}.
As mentioned in \refsec{intro}, \citet{sagawa2020group} demonstrated that stronger $L_2$-regularization can improve worst-group error on neural networks (when coupled with reweighting or group DRO). Similarly \citet{cao2019learning} show that data-dependent regularization can improve error on rare labels.
While their work focuses on developing methods to improve worst-group error, our focus is on understanding the mechanisms by which overparameterization hurts worst-group error.

\section{Discussion}\label{sec:discussion}
Our work shows that overparameterization hurts worst-group error on real datasets that contain spurious correlations.
We studied the implicit- and explicit-memorization settings to provide a potential story for why this might occur:
there can be an inductive bias towards solutions that do not need to memorize as many training points,
and this can favor models that exploit the spurious correlations.

However, our synthetic settings make several simplifying assumptions, e.g., they suppose that the model prefers the spurious feature because it is less noisy than the core feature. This assumption need not always apply, and different assumptions might also lead to overparameterization exacerbating spurious correlations. For example, there might exist a true classifier based on the core features which has high accuracy but which is relatively more complex (e.g., high parameter norm) and therefore not favored by the training procedure. Studying the effect of overparameterization in settings such as those is important future work.

We also observed that subsampling allows overparameterized models to achieve low average and worst-group test error, despite eliminating a large fraction of training examples. In contrast, when using the full training data, only underparameterized models attain low worst-group test error under our current training methods.
These observations call for future work to develop methods that can exploit both the statistical information in the full training data as well as the expressivity of overparameterized models, so as to attain good worst-group and average test error.
%
%
% While subsampling allows overparameterized models to obtain low worst-group test error, it throws away a large fraction of the training data.
% On the other hand, reweighting works with the full training data, but only achieves good worst-group error in the underparameterized regime.
% To obtain low worst-group error via these standard methods, we therefore need to choose between using expressive, overparameterized models or using the entire training data.
% Using both overparameterized models and the full training data improves average error,
% but there seems to be a tension between the two for worst-group error when using current approaches.


\subsection*{Acknowledgements}
We are grateful to Yair Carmon, John Duchi, Tatsunori Hashimoto, Ananya Kumar, Yiping Lu, Tengyu Ma, and Jacob Steinhardt for helpful discussions and suggestions. SS was supported by a Stanford Graduate Fellowship,
AR was supported by a Google PhD Fellowship and Open Philanthropy Project AI Fellowship, and PWK was supported by the Facebook Fellowship Program.

\subsection*{Reproducibiltity}
Code is available at \url{https://github.com/ssagawa/overparam_spur_corr}. All code, data, and experiments are available on the Codalab platform at \url{https://worksheets.codalab.org/worksheets/0x1db77e603a8d48c8abebd67fce39cf8b}.
\bibliography{refdb/all}
\bibliographystyle{icml2020}

\clearpage
\appendix
\onecolumn
\section{Supplemental experiments}
\subsection{ERM models have poor worst-group error regardless of the degree of overparameterization}
\label{sec:appendix_erm}
In the main text, we focused on reweighted models, trained with the reweighted objective on the full data
(Sections~\ref{sec:results_empirical}-\ref{sec:results_mechanism}),
as well as subsampled models, trained on subsampled data with the ERM objective (\refsec{results_subsampling}).
Here, we study the effect of overparameterization on ERM models, trained with the ERM objective on the full data.
Consistent with prior work, we observe that ERM models obtain poor worst-group error (near or worse than random), regardless of whether the model is underparameterized or overparameterized \citep{sagawa2020group}.
We also confirm that overparameterization helps average test error (see, e.g., \citet{nakkiran2019deep, belkin2019reconciling, mei2019generalization}).

\paragraph{Empirical results.}
We first consider the CelebA and Waterbirds dataset,
following the experimental set-up of \refsec{results_empirical} but now training with the standard ERM objective (Equation \refeqn{erm}) instead of the reweighted objective (Equation \refeqn{reweight}).

On these datasets, overparameterization helps the average test error (\reffig{erm_main}).
As model size increases past the point of zero training error,
the average test error decreases.
The best average test error is obtained by highly overparameterized models with zero training error---4.6\% for CelebA at width 96, and 4.2\% for Waterbirds at 6,000 random features.

In contrast, the worst-group error is consistently high across model sizes:
it is consistently worse than random ($>$50\%) for CelebA and nearly random (44\%) for Waterbirds (\reffig{erm_main}).
These worst-group errors are much worse than those obtained by reweighted, underparameterized models ($25.6\%$ for CelebA and $26.6\%$ for Waterbirds; see \refsec{results_empirical}).
Thus, while overparameterization helps ERM models achieve better test error, these models all fail to yield good worst-group error regardless of the degree of overparameterization.

\begin{figure*}[!h]
\centering
\includegraphics[width=1\textwidth]{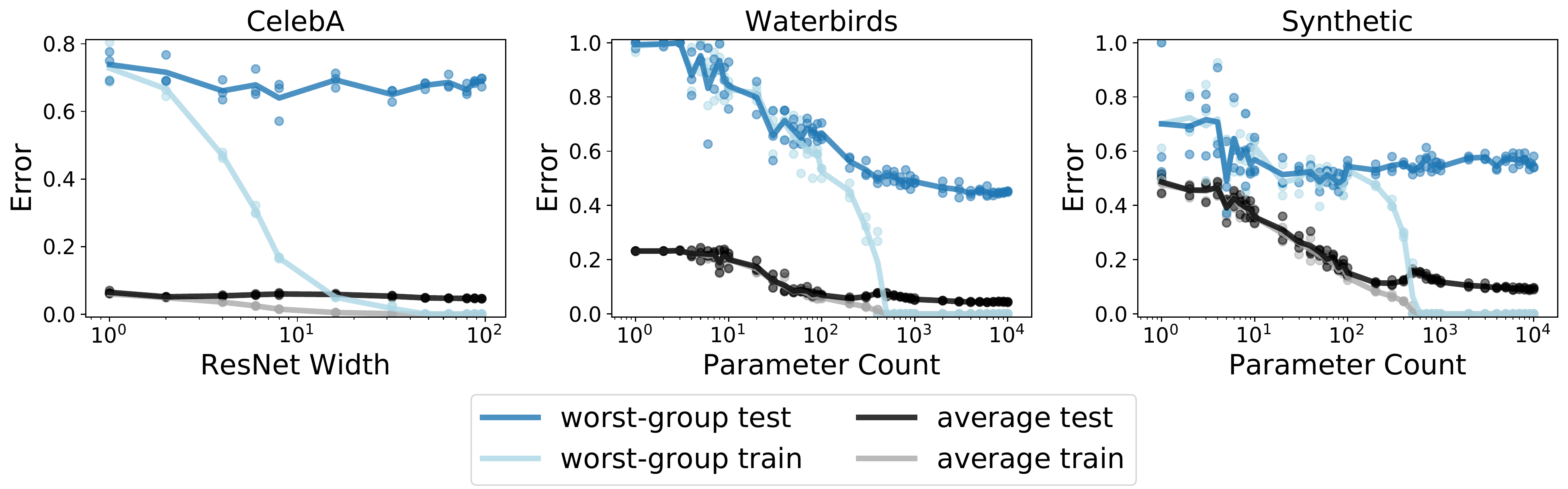}
  \caption{The effect of overparameterization on the average and worst-group error of an ERM model. Increasing model size helps average test error, but worst-group error remains poor across model sizes.}
  \label{fig:erm_main}
\end{figure*}

\paragraph{Simulation results.}
We also evaluate the effect of overparameterization on ERM models on the synthetic dataset introduced in \refsec{results_props}.
As above, ERM models fail to achieve reasonable worst-group test error across model sizes, but improve in average test error as model size increases (\reffig{erm_main}).
The best average test error is obtained by a highly overparameterized model with zero training error---9.0\% error at 9,000 random features---while the worst-group test error is nearly random or worse ($>48$\%) across model sizes.

\subsection{Stronger $L_2$ regularization improves worst-group error in overparameterized reweighted models}
\label{sec:appendix_reg}
In the main text, we studied models with default/weak or no $L_2$ regularization.
In this section, we study the role of $L_2$ regularization in modulating the effect of overparameterization on worst-group error by changing the hyperparameter $\lambda$ that controls $L_2$ regularization strength. Overall, we find that increasing $L_2$ regularization (to the point where models do not have zero training error) improves worst-group error but hurts average error in overparameterized reweighted models. In contrast, $L_2$ regularization has little effect on both worst-group and average error in the underparameterized regime.

\begin{figure*}[!b]
\centering
\vspace{-5mm}
\includegraphics[width=0.95\textwidth]{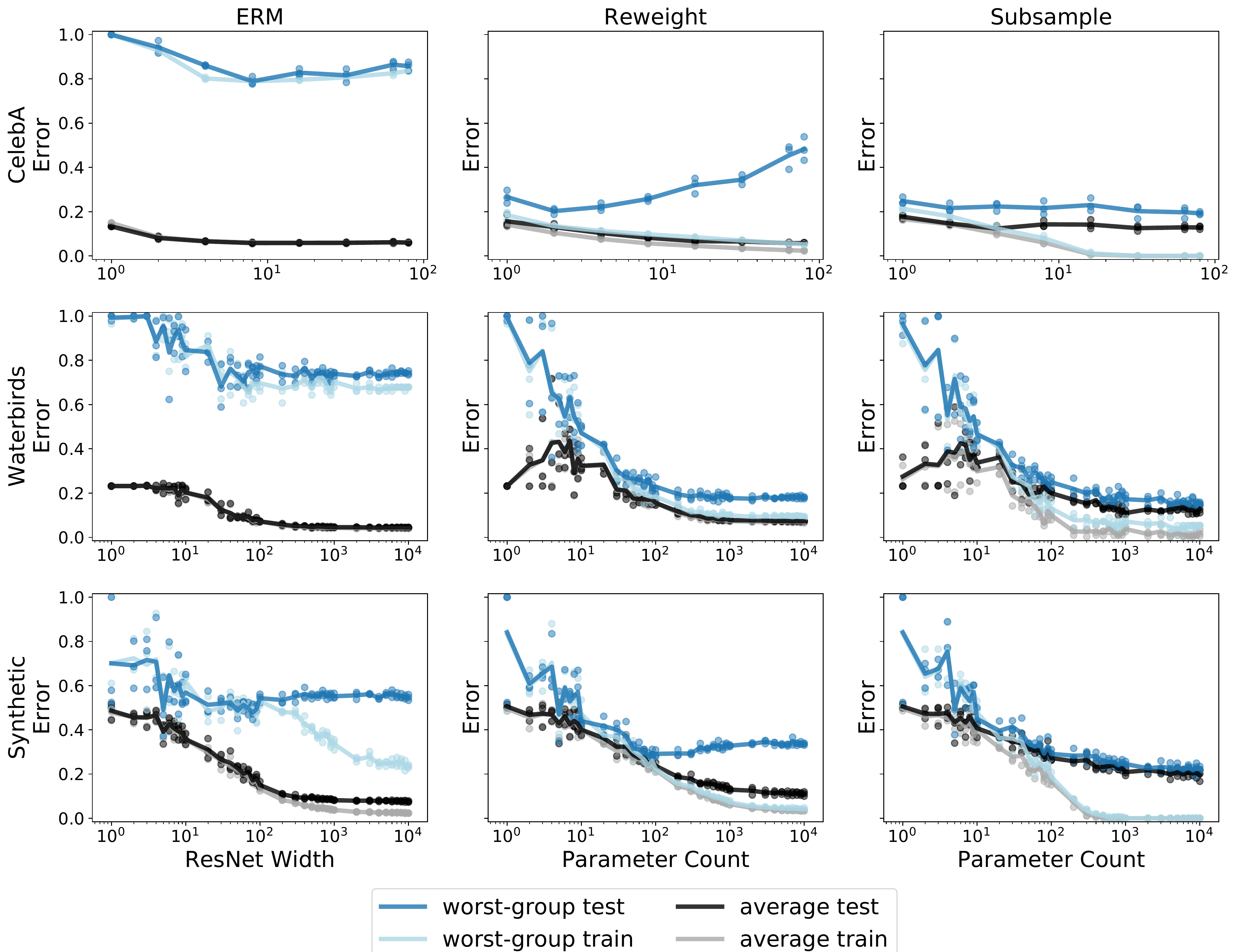}
  \vspace{-5mm}
  \caption{Strongly-regularized models have lower worst-group error than their weakly-regularized counterparts in the overparameterized regime (\reffig{image_curve}). Even under strong regularization, increasing model size can hurt the worst-group error on the CelebA (top) and synthetic (bottom) datasets, although overparameterization seems to improve worst-group error in the Waterbirds datase (middle) for the range of model sizes studied.}
  \label{fig:reg_main}
\end{figure*}

\paragraph{Strong $L_2$ regularization improves worst-group error in overparameterized reweighted models.}
In the main text, we trained ResNet10 models with default, weak regularization ($\lambda=0.0001$) on the CelebA dataset, and unregularized logistic regression on the Waterbirds and synthetic datasets.
Here, we consider strongly-regularized models with $\lambda=0.1$ for both types of models; unlike before, these models no longer achieve zero training error even when overparameterized.
\reffig{reg_main} shows the results of varying model size on strongly-regularized ERM, reweighted, and subsampled models on the three datasets.

On all three datasets, with strong regularization, ERM models continue to yield poor worst-group test error across model sizes, with similar or worse worst-group test error compared to with weak/ no regularization.
Conversely, strongly-regularized subsampled models continue to achieve low worst-group test error across model sizes.

Where strong regularization has a large effect is on reweighted models.
With reweighting, we find that strong regularization improves worst-group error in overparameterized models:
across all three datasets, the worst-group test error in the overparameterized regime is much lower for the strongly-regularized models than their weakly regularized or unregularized counterparts (\reffig{image_curve}).
These results are consistent with similar observations made in \citet{sagawa2020group}.
However, even though strongly-regularized overparameterized models outperform weakly-regularized overparameterized models, overparameterization can still hurt the worst-group error in strongly-regularized reweighted models.
On the CelebA and synthetic datasets, with $\lambda=0.1$, the best worst-group error is still obtained by an underparameterized model for the CelebA and synthetic datasets,
though overparameterization seems to help worst-group error on the Waterbirds dataset at least in the range of model sizes studied.

\paragraph{Overparameterized models require strong regularization for worst-group test error but not average test error.}
Given a fixed overparameterized model size, how does its performance change with the $L_2$ regularization strength $\lambda$?
We study this with the logistic regression model on the Waterbirds and synthetic datasets, using
a model size of $m=10,000$ random features and varying the $L_2$ regularization strength from $\lambda=10^{-9}$ to $\lambda=10^2$.
\footnote{We did not run this experiment on the CelebA dataset for computational reasons, as doing so would have required tuning a different learning rate for each choice of regularization strength.}

Results are in \reffig{error_vs_reg}.
As before, ERM models obtain poor worst-group error regardless of the regularization strength,
and subsampled models are relatively insensitive to regularization, achieving reasonable worst-group error at most settings of $\lambda$.

For reweighted models, however, having the right level of regularization is critical for obtaining good worst-group test error. On both datasets, the best worst-group test error is obtained by strongly-regularized models that do not achieve zero training error. In contrast, increasing regularization strength hurts average error, with the best average test error attained by models with nearly zero regularization.

\begin{figure*}[!h]
\centering
\includegraphics[width=1\textwidth]{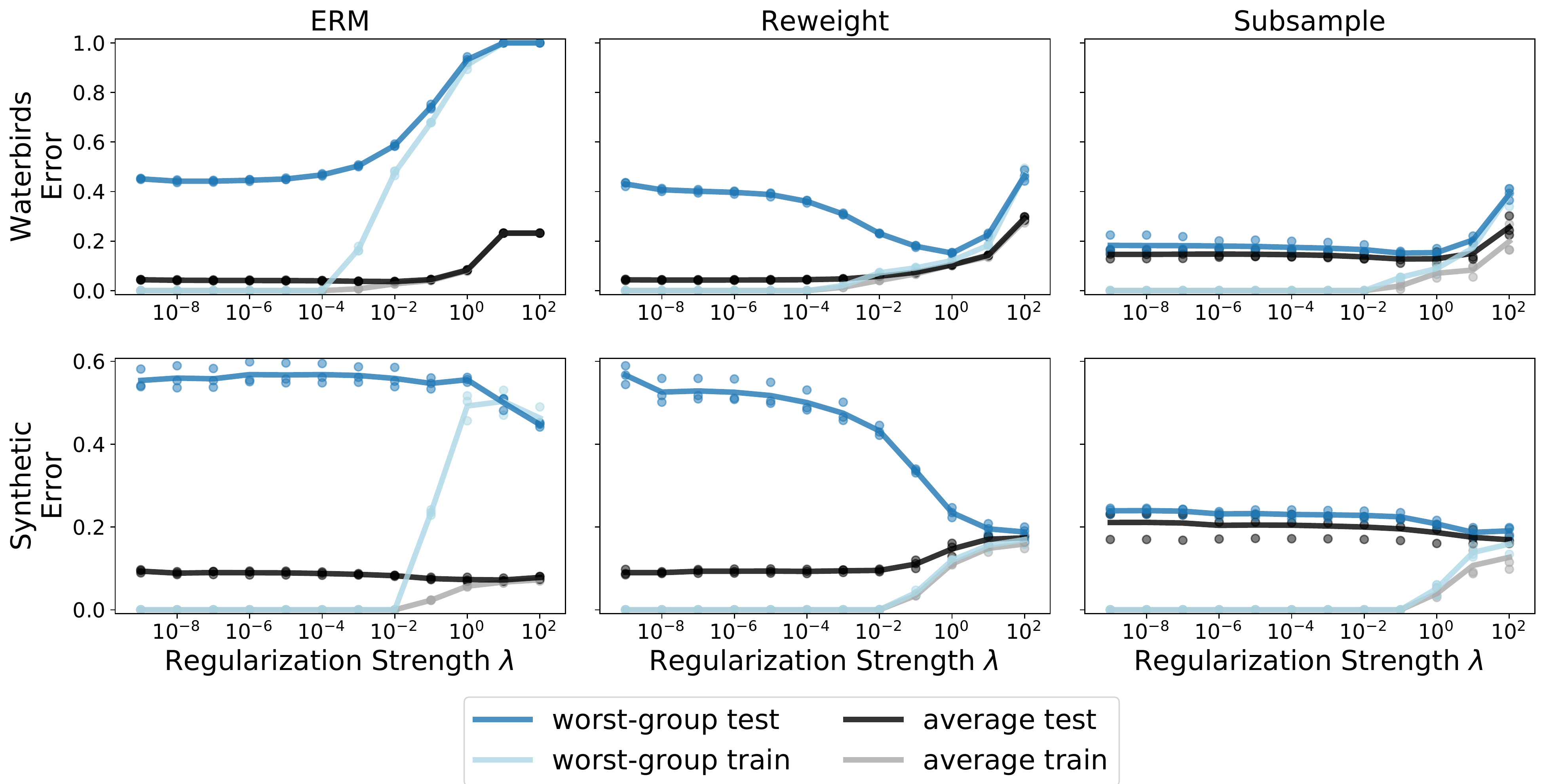}
  \caption{The effect of regularization on overparameterized random features logistic regression models ($m=10,000$). ERM models (left) do consistently poorly while subsampled models (right) do consistently well on worst-group error. For reweighted models (middle), the best worst-group error is obtained by a strongly-regularized model that does not achieve zero training error.}
  \label{fig:error_vs_reg}
\end{figure*}

\paragraph{$L_2$ regularization affects where worst-group test error plateaus as model size increases.}
In the above experiments, we kept either model size or regularization strength fixed, and varied the other. Here, we vary both: we consider $L_2$ regularization strengths $\lambda \in \{10^{-9}, 10^{-6}, 0.001, 0.1, 10\}$ and investigate the effect of increasing model size for each $\lambda$.
We plot the results for Waterbirds and the synthetic dataset in \reffig{reg_waterbirds} and \reffig{reg_synthetic} respectively.

For reweighted models, the results match what we observed above.
Strengthening $L_2$ regularization reduces the detrimental effect of overparameterization on worst-group error.
For any fixed model size in the overparameterized regime, the worst-group test error improves as $\lambda$ increases up to a certain value.
Worst-group test error seems to plateau at different values as model size increases, depending on the regularization strength,
though we note that it is possible that further increasing model size beyond the range we studied might lead models with different regularization strengths to eventually converge.
Further empirical studies as well as theoretical characterization of the interaction between regularization and overparameterization are needed to confirm this phenomenon.

Given sufficiently large $\lambda$ (e.g., $\lambda=10$ for both Waterbirds and synthetic datasets),
overparameterized models seem to outperform underparameterized models, at least for the range of model sizes studied.
However, we caution that this trend does not seem to hold on the CelebA dataset (\reffig{reg_main}).

Finally, in contrast with its effects on overparameterized models, regularization seems to only have a modest effect on worst-group test error in the underparameterized regime.

\begin{figure*}[!h]
\centering
\includegraphics[width=1\textwidth]{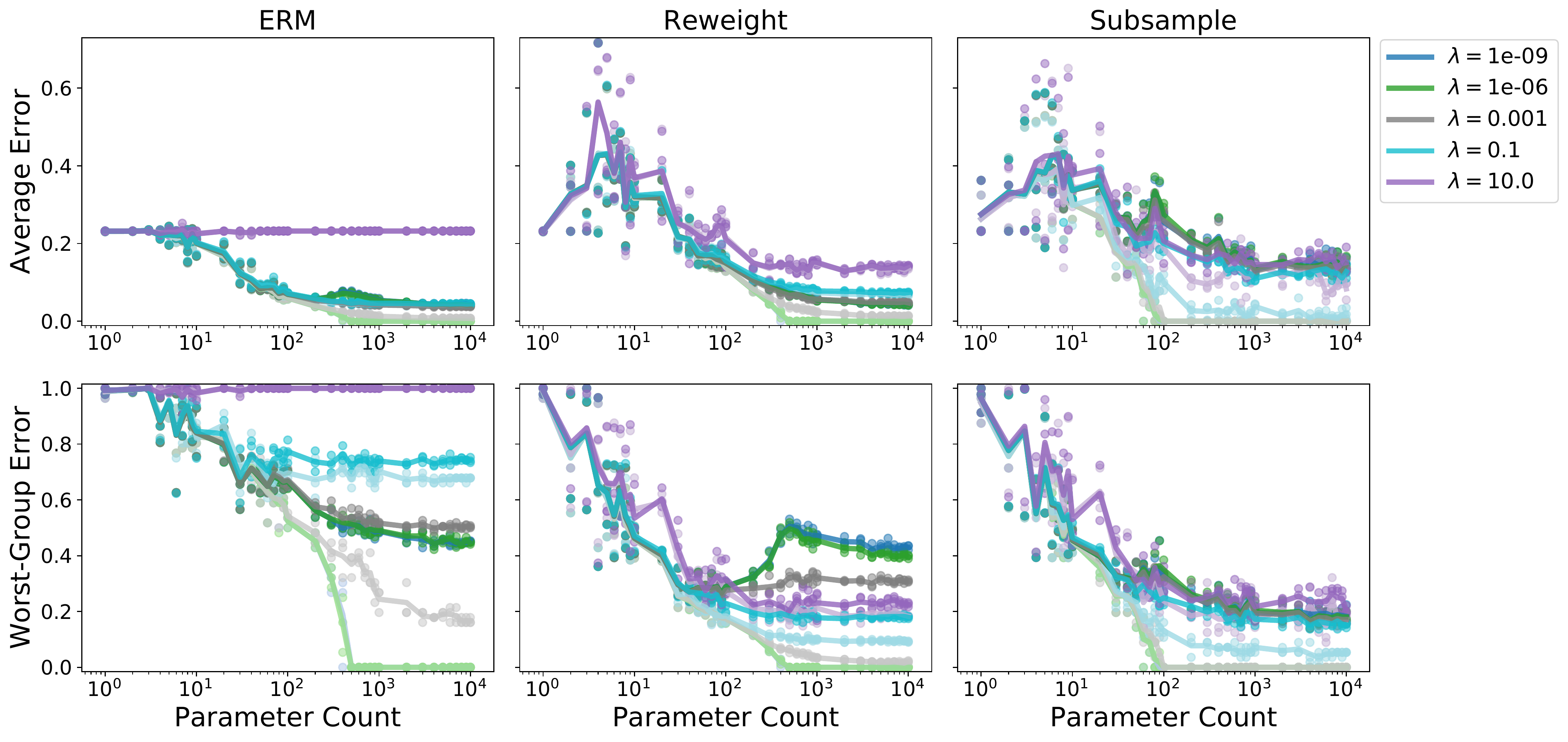}
  \caption{The effect of overparameterization on models with different $L_2$ regularization strengths $\lambda$ on the Waterbirds dataset. Different regularization strengths are shown in different colors, with training and test errors plotted in light and dark colors, respectively.}
  \label{fig:reg_waterbirds}
\end{figure*}

\begin{figure*}[!h]
\centering
\includegraphics[width=1\textwidth]{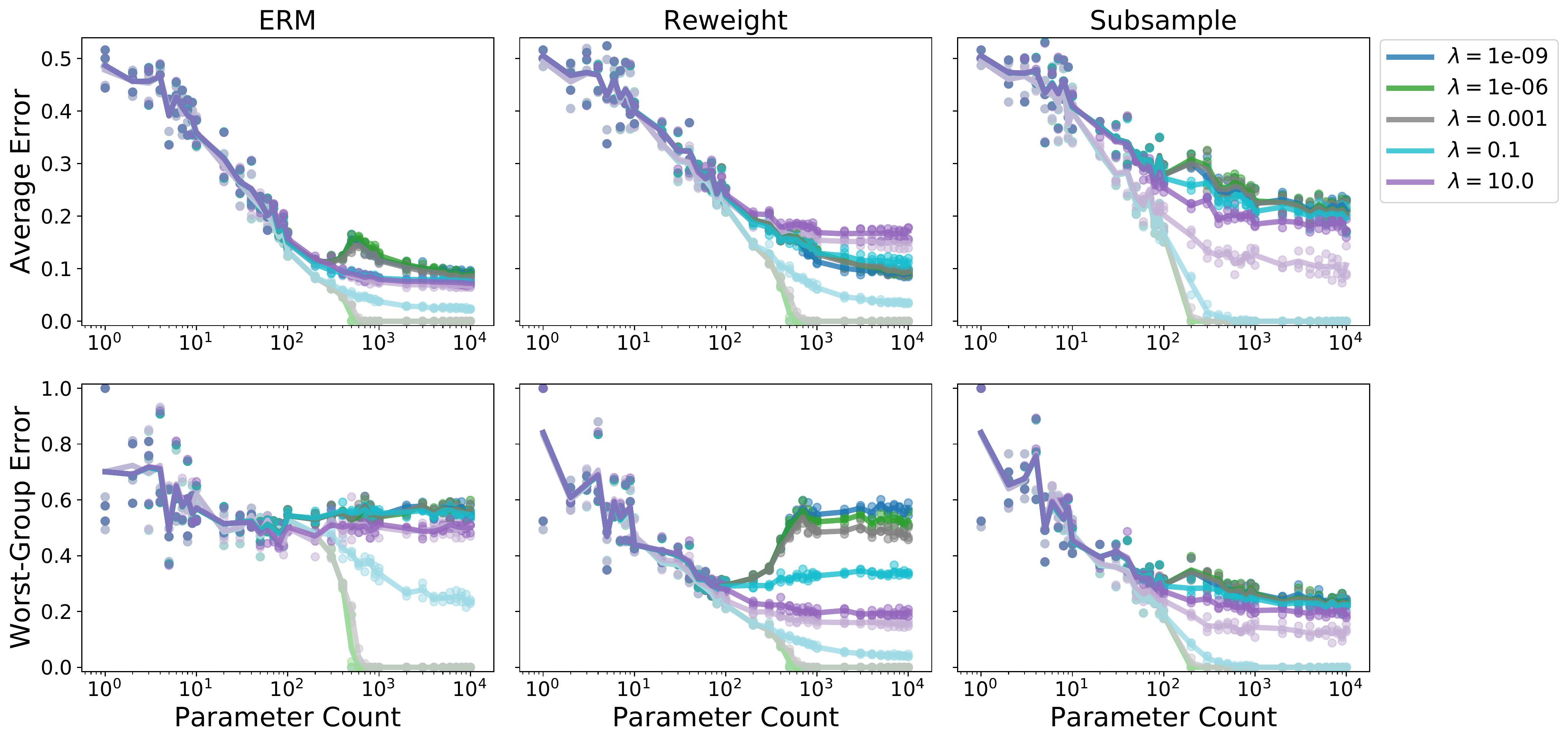}
  \caption{The effect of overparameterization on models with different $L_2$ regularization strengths $\lambda$ on the synthetic dataset. The plotting scheme follows that of \reffig{reg_waterbirds}.}
  \label{fig:reg_synthetic}
\end{figure*}

\newpage
\subsection{Overparameterization helps average test error on the synthetic data regardless of $\pmajmath$ and $\scrmath$}
\label{sec:appendix_avg}
\reffig{toy_knob_avg} shows how the average test error changes as a function of model size under different settings of the majority fraction $\pmajmath$ and the spurious-core ratio $\scrmath$ on the synthetic dataset introduced in \refsec{results_props}.
As expected, overparameterization helps the average test error regardless of SCR and the majority fraction.
\vspace{-3mm}
\begin{figure*}[!h]
\centering
\includegraphics[width=0.7\textwidth]{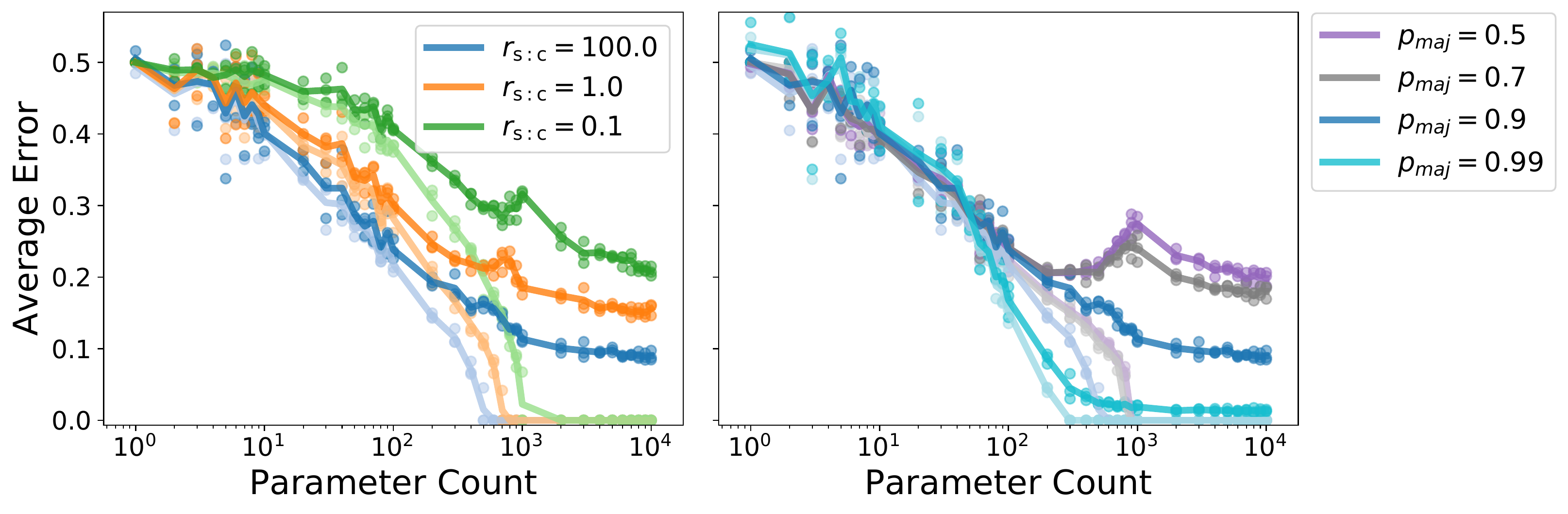}
  \vspace{-3mm}
  \caption{The effect of overparameterization on average error of a reweighted model on synthetic  data. Different values of $\pmajmath$ and $\scrmath$ are plotted in different colors, with training and test errors plotted in light and dark colors, respectively. Across all values of $\pmajmath$ and $\scrmath$, overparameterization helps the average test error.}
  \vspace{-3mm}
  \label{fig:toy_knob_avg}
\end{figure*}

\subsection{Comparison between implicit and explicit implicit memorization}\label{sec:memorization_comparison}

To motivate the explicit-memorization setting, we ran some brief experiments to show that in the overparameterized regime, linear models in the explicit-memorization setting behave similarly to random projection (RP) models in the implicit-memorization setting, with $\sigcau^2$ and $\sigspu^2$ in the latter scaled up by a factor of $d$ (\reffig{linear}).
Recall that in the latter, $\xcau \in \R^d$ is distributed as $\xcau | y \sim \sN(y, \sigcau^2 I_d)$.
Roughly speaking, all the information about $y$ is contained in the mean $\bar{x}_\text{core} = \frac{1}{d} \sum_j x_{\text{core},j}$, which is distributed as $\sN(y, \sigcau^2 I_d/d)$.
In the explicit-memorization setting, we can view $\xcau \in \R$ as equivalent to $\bar{x}_\text{core}$ in the implicit-memorization setting (and similarly for $\xspu$), explaining the quantitative fit observed in \reffig{linear}.

However, in the highly underparameterized regime, the RP models do poorly because of model misspecification (owing to a small number of random projections), whereas the linear models can still learn to use $\xcau$ and therefore do well.

\begin{figure}[!ht]
\centering
\includegraphics[width=0.6\textwidth]{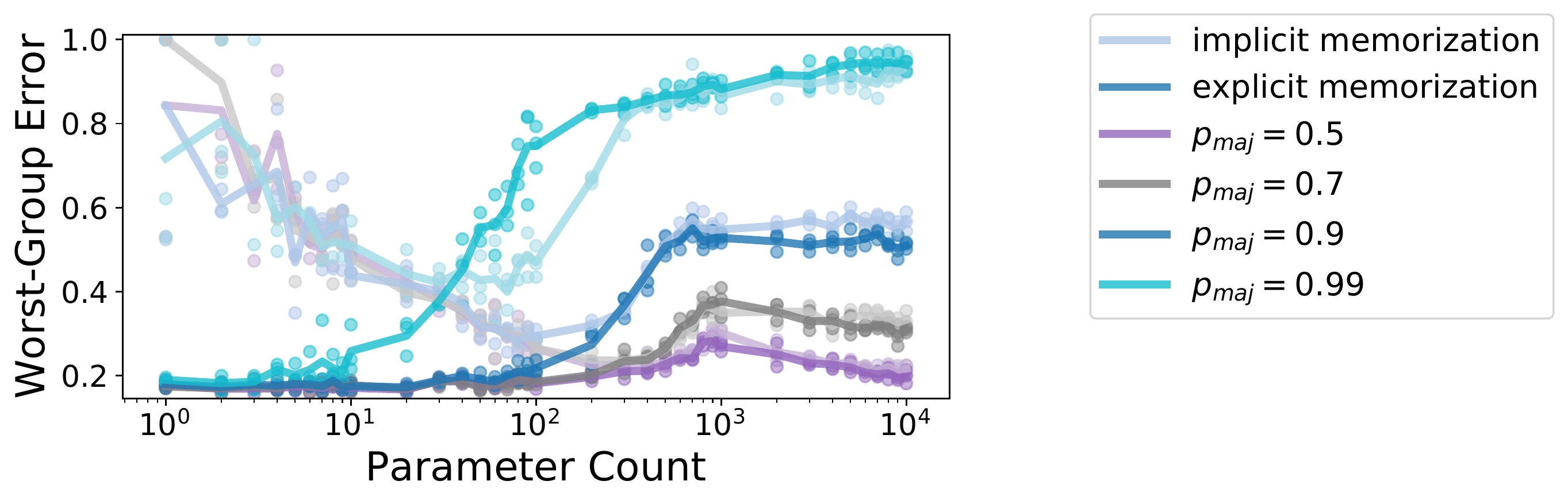}
  \vskip -3mm
  \caption{The effect of overparameterization on the \wg test error for linear models in the explicit-memorization setting ($\sigcau^2 = 1, \sigspu^2 = 0.01, \signoise^2 = 1$) and random projection models in the implicit-memorization setting ($\sigcau^2 = 100, \sigspu^2 = 1, d=100$). The models agree in the overparameterized regime.
  }
  \label{fig:linear}
\end{figure}

\subsection{Experimental details}
\label{sec:appendix_details}
\paragraph{Waterbirds and CelebA datasets.}
For the CelebA dataset, we use the official train-val-test split from \citet{liu2015deep},
with the \emph{Blond\_Hair} attribute as the target $y$ and the \emph{Male} as the spurious association $a$.

For the Waterbirds dataset, we follow the setup in \citet{sagawa2020group}; for convenience, we reproduce some details of how it was constructed here.
This dataset was obtained by combining bird images from the CUB dataset \citep{wah2011cub} with backgrounds from the Places dataset \citep{zhou2017places}. The CUB dataset comes with annotations of bird species. For the Waterbirds dataset, each bird was labeled was a waterbird if it was a seabird or waterfowl in the CUB dataset; otherwise, it was labeled as a landbird. Bird images were cropped using the provided segmentation masks and placed on either a land (bamboo forest or broadleaf forest) or water (ocean or natural lake) background obtained from the Places dataset.

For Waterbirds, we follow the same train-val-test split as in \citet{sagawa2020group}. Note that in these validation and test sets, landbirds and waterbirds are uniformly distributed on land and water backgrounds so that accuracy on the rare groups can be more accurately estimated.
When calculating average test accuracy, we therefore first compute the average test accuracy over each group and then report a weighted average, with weights corresponding to the relative proportion of each group in the skewed training dataset.

We post-process Waterbirds by extracting feature representations taken from the last layer of a ResNet18 model pre-trained on ImageNet. We use the Pytorch \texttt{torchvision} implementation of the ResNet18 model for this. All models on the Waterbirds dataset in our paper are logistic regression models trained on top of this (fixed) feature representation.

\paragraph{ResNet.}
We used a modified ResNet10 with variable widths, following the approach in \citet{nakkiran2019deep} and extending the \texttt{torchvision} implementation.
We trained all ResNet10 models with stochastic gradient descent with momentum of 0.9 and a batch size of 128, with the $L_2$ regularization parameter $\lambda$ was passed in to the optimizer as the weight decay parameter. In the experiments in the main text, we used the default setting of $\lambda = 10^{-4}$.
We used a fixed learning rate instead of a learning rate schedule and selected the largest learning rate for which optimization was stable, following \citet{sagawa2020group}.
This resulted in learning rates of 0.01 and 0.0001 for $\lambda=10^{-4}$ and $\lambda=0.1$, respectively, across all training procedures.
As in the original ResNet paper \citep{he2016resnet}, we used batch normalization \citep{ioffe2015batch} and no dropout \citep{srivastava2014dropout}, and for simplicity, we trained all models without data augmentation.

We trained for 50 epochs for ERM and reweighted models and 500 epochs for subsampled models (due to smaller number of examples per epoch).
We found that worst-group error can be unstable across epochs due to the small sample size and relatively large learning rate, so in our results we report the error averaged over the last 10 epochs.

\paragraph{Logistic regression.}
We used the logistic regression implementation from \texttt{scikit-learn}, training with the L-BFGS solver until convergence with tolerance 0.0001, and setting the regularization parameter as $C=1/(n\lambda)$. For unregularized models, we set $\lambda=10^{-9}$ for numerical stability.

\subsection{Subsampling}\label{sec:appendix_subsampling}

Formally, given a set of groups $\sG$ and a dataset $\text{D}$ comprising a set of $n$ training points with their group identities $\{ (\ix{i}, \iy{i}, \ig{i}) \}$, the subsampling procedure involves two steps.
First,
we group training points based on group identities:
  \begin{align}
    \text{D}_g &\eqdef \{ (\ix{i}, \iy{i}) \mid \ig{i} = g \} ~ \text{for each } g \in \sG.
  \end{align}
  For each group $g$, we select a subset $\dgss \subseteq \text{D}_g$ uniformly at random from $\text{D}_g$ such that each subset has the same number of points as the smallest group in the training set. We form a new dataset $\dss$ by combining these subsets:
  \begin{align}
    \dss &= \bigcup \limits_{g \in \sG} \dgss, ~\text{where}\\
    \dgss \subseteq \text{D}_g ~~&\text{and}~~| \dgss | = \min \limits_{g \in \sG} | \text{D}_g |\nonumber
  \end{align}
 Note that $\dss$ is group-balanced, with $\pmajmath = 0.5$.
We then train a model by minimizing the average loss on $\dss$,
  \begin{align}
    \label{eqn:subsample}
    \hat{\mathcal{R}}_\mathsf{subsample}(w) \eqdef \frac{1}{| \dss |} \sum \limits_{(x, y) \in \dss} \ell(w; (x, y)).
  \end{align}
Since $\dss$ is group-balanced, the reweighted training loss (Equation~\ref{eqn:reweight}) has the same weight on all training points and minimizing the reweighted objective on $\dss$ is equivalent to minimizing the average loss objective above.

\section{Proof of \refthm{main}}
\label{sec:app-analysis}
Here, we detail the proof of \refthm{main} presented in \refsec{results_mechanism}.
We structure the proof by splitting \refthm{main} into two smaller theorems:
one for the overparameterized regime (\refapp{app-over-analysis}),
and another for the underparameterized regime (\refapp{app-under}).

\subsection{Notation and definitions.}\label{sec:app-notation}
We denote the separate components of the weight vector $\cau{\hat{w}} \in \R, \spu{\hat{w}} \in \R, \noise{\hat{w}} \in \R^N$ such that
\begin{align}
\hat{w} = [\cau{\hat{w}}, \spu{\hat{w}}, \noise{\hat{w}}].
\end{align}
Further, by the representer theorem, we decompose $\noise{\hat{w}}$ as
\begin{align}
\noise{\hat{w}} = \sum \limits_{i=1}^n \is{\hat{w}}{i} \ixnoi{i}.
\end{align}
Note that $\is{w}{i}$ is equivalent to the $\isw{w}{i}$ referred to in the main text. Recall that we define memorization of each training point $\ix{i}$ by the weight $\isw{w}{i}$ as follows.

\begin{definition}[$\gamma$-memorization]
Consider a separator $\hat{w}$ on training data $\{(\ix{i}, \iy{i})\}_{i=1}^n$.
For some constant $\gamma \in \R$, we say that a model $\gamma$-memorizes a training point if
\begin{align}
\left|\is{\hat{w}}{i}\right| > \frac{\gamma^2}{\signoise^2}.
\end{align}
\end{definition}
The component $\is{\hat{w}}{i} \ixnoi{i}$ serves to ``memorize'' $\ix{i}$ when $\nnoise$ is sufficiently large, as it affects the prediction on $\ix{i}$ but not on any other training or test points (because noise vectors are nearly orthogonal when $\nnoise$ is large). In the proof, we set the constant $\gamma^2$ appropriately (based on other parameter settings in Theorem~\ref{thm:main}) to get the required result.

Finally, let $\gmaj, \gmin$ denote the indices of training points in the majority and minority group respectively.

\subsection{Overparameterized regime}
\label{sec:app-over-analysis}
In our explicit-memorization set-up, sufficiently overparameterized models provably have high worst-group error under certain settings of $\sigspu^2, \sigcau^2,\nmaj,\nmin$ as stated in \refthm{main} (restated below as \refthm{over}).

\begin{restatable}[]{thm}{over}
  \label{thm:over}
  For any $\pmajmath \geq \bigl(1-\frac{1}{2001}\bigr)$, $\sigcau^2 \geq 1$, $\sigspu^2 \leq \frac{1}{16 \log 100 \nmaj}$, $\signoise^2 \leq \frac{\nmaj}{600^2}$ and $\nmin \geq 100$,
  there exists $\nnot$ such that for all $\nnoise > \nnot$ (overparametrized regime), with high probability over draws of the data,
  \begin{align}
    \roberr{\estmmprimal} \geq 2/3,
  \end{align}
  where $\estmmprimal$ is the max-margin classifier.
\end{restatable}

In \refsec{results_mechanism}, we sketched key ideas in the proof by considering special families of separators:
because the minimum-norm inductive bias favors less memorization, models can prefer to learn the spurious feature and memorize the minority examples (entailing high worst-group error), instead of learning the core feature and memorizing some fraction of all training points (possibly attaining reasonable worst-group error).
We now provide the full proof of \refthm{over}, generalizing the above key concepts by considering \emph{all} separators.

\begin{proof}
Recall from \refsec{results_mechanism} that we consider the maximum-margin classifier $\estmm$:
\begin{align}
  \estmm = \arg\min \| w \|_2^2~~\text{s.t.}~~y^{(i)}(w \cdot x^{(i)}) \geq 1, ~\forall i.
\end{align}
In other words, $\estmm$ is the minimum-norm separator, where separator is a classifier with zero training error and required margins, satisfying $y^{(i)}(w \cdot x^{(i)}) \geq 1$ for all $i$.
We analyze the worst-group error of the minimum-norm separator $\estmm$ as outlined below:
  \begin{enumerate}
  \item
    We first upper bound the fraction of \emph{majority} examples memorized by the minimum-norm separator $\estmm$.
    We show that there exists a separator that can use spurious features and needs to memorize only the minority points (\reflem{upperbound}) for the parameter settings in Theorem~\ref{thm:over} where $\sigspu$ is sufficiently small.
    Since the norm of a separator is roughly scales with the number of points memorized ($|\is{\hat{w}}{i}| \geq \gamma^2/\signoise^2$), we have an upper bound on the number of training points memorized by $\estmm$. Since the number of majority points is much larger than the number of minority points, this says that only a small fraction of majority points could be memorized by $\estmm$.
  \item
    Next, we observe that since the core feature is noisy as per the parameter setting in Theorem~\ref{thm:over}, if we do not use the spurious feature, a constant fraction of majority points have to be memorized if spurious features are not used. Conversely, if less than this fraction of majority points can be memorized, the separator must use spurious features.
    Since using spurious features leads to higher worst-group test error, this reveals a trade-off between the worst-group test error of a separator and the fraction of \emph{majority points} that it memorizes at training time. Succinctly, smaller fraction memorized implies the use of spurious features which in turn implies higher worst-group test error. Smaller worst-group test error requires eliminating the use of spurious features which would lead to a large fraction of majority points requiring memorization in order for a classifier to be a separator.
We formalize the above trade-off between the worst-group test error and fraction of majority examples to be memorized in Proposition~\ref{prop:fracmaj-roberr}.
  \end{enumerate}
Combining the two steps together, since $\estmm$ memorizes only a small fraction of majority points by virtue of being the minimum norm separator, $\estmm$ suffers high worst-group test error.

We now formally prove \refthm{over}, invoking propositions that we prove in subsequent sections.

\subsubsection{Bounding the fraction of memorized examples in the majority groups.}
In the first part of the proof, we show that the minimum-norm separator $\estmm$ ``memorizes'' a small fraction of the majority examples. Formally, we study the quantity $\deltrain{\hat{w}, \gamma^2}$ defined as follows.
\begin{definition}
\label{def:deltrain}
Consider a separator $\hat{w}$ on training data $\{(\ix{i}, \iy{i})\}_{i=1}^n$.
Let $\deltrain{\hat{w}, \gamma^2}$ be the fraction of training examples that $\hat{w}$ $\gamma$-memorizes in the majority groups:
\begin{align}
\deltrain{\hat{w}, \gamma^2} \eqdef \frac{1}{\nmaj} \sum_{i\in\gmaj} \mathbb{I}\left[\left|\is{\hat{w}}{i}\right| > \frac{\gamma^2}{\signoise^2}\right]
\end{align}
\end{definition}
We provide an upper bound on $\deltrain{\estmm,\gamma^2}$ (\reflem{deltrain}) by first bounding $\|\estmm\|$ and then bounding $\deltrain{\estmm,\gamma^2}$ in terms of $\|\estmm\|$.

\underlinepara{Bounding $\|\estmm\|$}
\begin{lemma}
  \label{lem:upperbound}
  There exists a separator $\memmin$ that satisfies $\iy{i}(\memmin \cdot \ix{i}) \geq 1, ~\forall i \in \gmaj, \gmin$. The norm of this separator gives a bound on $\| \estmm \|$ as follows. For the parameter settings under Theorem~\ref{thm:over}, with high probability, we have
\begin{align}
  \| \estmm \|_2^2 \leq \| \memmin \|_2^2 \leq u^2 + s^2 \signoise^2 (1 + c_1) \nmin + \frac{s^2 \signoise^2}{n^4},
\end{align}
for constants $u = 1.3125, s = \frac{2.61}{\signoise^2}$.
\end{lemma}
\begin{proof}
In order to get an upper bound on $\| \estmm \|$, we compute the norm of a particular separator. Concretely, we consider a separator $\memmin$ of the following form:
  \begin{align*}
  \cau{\memmin} &= 0\\
  \spu{\memmin} &= u\\
  \noise{\memmin} &= \sum \limits_i \is{\memmin}{i} \ixnoi{i}\\
  \is{\memmin}{i}&=0 \text{ for } i \in \gmaj\\
  \is{\memmin}{i}&=\iy{i} s \text{ for } i\in\gmin
  \end{align*}
  First, because we are interested in $\memmin$ that does not use the core feature and relies on the spurious feature instead, we let $\cau{\memmin} = 0$ and $\spu{\memmin} = u, u \in \R$. We set the value $u$ appropriately so that none of the majority points are memorized (corresponding to $\is{\memmin}{i} = 0$ for all $i\in\gmaj$). However since the spurious correlations are reversed in the minority points and $\cau{\memmin} = 0$, the minority points have to be memorized. For simplicity, we set $\is{\memmin}{i} = \iy{i}s$ for all $i\in\gmin$.

  Now it remains to select appropriate values of constants $u$ and $s$ such that $\iy{i}  (\memmin \cdot \ix{i}) \geq 1$ is satisfied for all training examples.

  For majority points, this involves setting $u$ large enough such that the less noisy spurious feature can be used to obtain the required margin. Without loss of generality, assume $\iy{i} = 1$. Formally, for $i \in \gmaj$,
  \begin{align*}
    \memmin \cdot \ix{i} &\geq \ixspu{i} u + \sum \limits_{j \in \gmin} s \ixnoi{i} \cdot \ixnoi{j} \\
    &\geq 4/5 u + \sum \limits_{j \in \gmin} s \ixnoi{i} \cdot \ixnoi{j}, ~\text{w.h.p. from Lemma~\ref{lem:allspubound} with $a = y = 1$} \\
    &\geq 4/5 u - \frac{s \signoise^2}{n^5}, ~\text{w.h.p. from Lemma~\ref{lem:alldotp}}. \\
    &\geq 4/5 u - \frac{s \signoise^2}{100}.
  \end{align*}
  The first inequality follows from the fact that $\sigspu$ is small enough under the parameter settings of Theorem~\ref{thm:over} to allow a uniform bound on $\ixspu{i}$ (Lemma~\ref{lem:allspubound}). The second inequality follows from setting the number of random features $N$ to be large enough so that the noise features are near orthogonal (Lemma~\ref{lem:alldotp}). Conversely, we have
  \begin{align}
    \label{eq:separate-maj}
    4/5 u - \frac{s \signoise^2}{100} \geq 1 \implies \text{$\memmin$ is a separator on the majority points w.h.p}.
  \end{align}

  Notice that the condition in Equation~\ref{eq:separate-maj} requires that $u$ be greater than $0$. Since the minority points have spurious attribute $a = -y$, we need to set $s$ to be large enough so that $\memmin$ as defined above separates the minority points. Just as before, we set $y=1$ WLOG. For $ i \in \gmin$, we have
  \begin{align*}
    \memmin \cdot \ix{i} &\geq \ixspu{i} u + \sum \limits_{j \in \gmin} s \ixnoi{i} \cdot \ixnoi{j} \\
    &\geq -6/5 u + \sum \limits_{j \in \gmin} s \ixnoi{i} \cdot \ixnoi{j}, ~\text{From Lemma~\ref{lem:allspubound} with $a = -y = -1$} \\
    &\geq -6/5 u + s(1 - c_1) \signoise^2 - \frac{s \signoise^2}{n^5}, ~\text{w.h.p from Lemma~\ref{lem:alldotp} and Lemma~\ref{lem:allnorm}} \\
    &\geq -6/5 u + s(1 - c_1) \signoise^2 - \frac{s \signoise^2}{100}.
  \end{align*}
  The steps are similar to the condition for majority points, with the key difference that the contribution from the noise term involves $s \| \ixnoi{i} \|_2^2$ (Lemma~\ref{lem:allnorm}).

Conversely, we have
\begin{align}
  \label{eq:separate-min}
  -6/5 u + s(1 - c_1) \signoise^2 - \frac{s \signoise^2}{100} \geq 1 \implies \text{$\memmin$ is a separator on the minority points w.h.p.}.
\end{align}

A set of parameters that satisfies both conditions above Equation~\ref{eq:separate-min} and Equation~\ref{eq:separate-maj} is the following: $$u = 1.3125, s \signoise^2 = 2.61.$$ We use the fact that $c_1 < 1/2000$ (From Lemma~\ref{lem:allnorm}).

Finally, we have w.h.p,
\begin{align}
  \label{eq:upperbound}
  \| \memmin \|_2^2 \leq u^2 + s^2 \signoise^2 (1 + c_1) \nmin + \frac{s^2 \signoise^2}{n^4}.
\end{align}
This follows from bounds on $\| \ixnoi{i} \|_2^2$ (Lemma~\ref{lem:allnorm}) and sum of less than $n^2$ terms involving $s^2 \ixnoi{i} \cdot \ixnoi{j}$ (using Lemma~\ref{lem:alldotp}).
\end{proof}

\underlinepara{Bounding $\deltrain{\hat{w},\gamma^2}$ in terms of $\|\hat{w}\|$}

\begin{lemma}
\label{lem:norm_bound_delta}
For a separator $\hat{w}$ with bounded $\is{\hat{w}}{i}^2 \le \frac{10n}{\signoise^2}$ for all $i=1,\dots,n$, its norm can be bounded with high probability as
\begin{align}
 \| \hat{w} \|_2^2 &\geq \frac{\gamma^4 (1 - c_1)}{\signoise^2}\deltrain{\hat{w}, \gamma^2}\nmaj - \frac{10}{\signoise^2 n^3}
\end{align}
under the parameter settings of \refthm{over}.
\end{lemma}
\begin{proof}
The result follows bounded norms (\reflem{allnorm}), bounded dot products (\reflem{alldotp}), and the definition of $\deltrain{\hat{w},\gamma^2}$ (\refdef{deltrain}).
\begin{align}
  \| \hat{w} \|_2^2 &\geq \sum \limits_{i \in \gmaj} {\is{\hat{w}}{i}}^2 \| \ixnoi{i} \|_2^2 + \sum \limits_{j\neq k} \is{\hat{w}}{j} \is{\hat{w}}{k} \ixnoi{j} \cdot \ixnoi{k} \\
  &\geq \underbrace{\Big(\frac{\gamma^4 (1 - c_1) }{\signoise^2}\Big)\deltrain{\hat{w}, \gamma^2}\nmaj}_{\text{Choosing only points with $\is{\hat{w}}{i} \geq \gamma^2/\signoise^2$}} - \underbrace{\frac{M^2}{\signoise^2 n^4}}_{\text{$\max \is{\hat{w}}{i} = M/\signoise^2$}}, ~\text{w.h.p.} \\
    &\geq \frac{\gamma^4 (1 - c_1)}{\signoise^2}\deltrain{\hat{w}, \gamma^2}\nmaj - \frac{10}{\signoise^2 n^3}
\end{align}
\end{proof}

\underlinepara{Bounding $\deltrain{\estmm, \gamma^2}$}
~\\
We now apply \reflem{upperbound} and \reflem{norm_bound_delta} in order to bound $\deltrain{\estmm,\gamma^2}$,
showing that the fraction of majority points that are memorized is small for appropriate choice of $\gamma$.

To invoke \reflem{norm_bound_delta}, we first show that the coefficient $\is{\estmm}{i}$ is bounded above with high probabiltity.
\begin{lemma}
  \label{lem:supperbound}
  Under the parameter settings of \refthm{over}, with high probability, $\is{\estmm}{i}$ is bounded above for $i=1,\dots,n$ as
  \begin{align}
    {\is{\estmm}{i}}^2 \leq \frac{10 n}{\signoise^4}.
  \end{align}
\end{lemma}
\begin{proof}
  Let $\max \limits_i \is{\estmm}{i} = \frac{M}{\signoise^2}$.
  \begin{align}
    \| \estmm \|_2^2
    &\geq \| \noise{\estmm} \|_2^2\\
    & = \sum \limits_{i \in \gmin \gmaj} {\is{\estmm}{i}}^2 \| \ixnoi{i} \|_2^2 + \sum \limits_{i, j} \is{\estmm}{i} \is{\estmm}{j} \ixnoi{i} \cdot \ixnoi{j}\\
    &\geq \frac{M^2 (1 - c_1)}{\signoise^2}  - \frac{M^2}{\signoise^2 n^6}n^2 \\
    &\geq \frac{M^2 (1 - c_1)}{\signoise^2}  - \frac{M^2}{\signoise^2 n^4}.
  \end{align}
  From the upper bound on $\| \estmm \|_2^2$ (\reflem{upperbound}), we have
  \begin{align}
   & \frac{M^2 (1 - c_1)}{\signoise^2}  - \frac{M^2}{\signoise^2 n^4} \leq u^2 + s^2 \signoise^2 (1 + c_1) \nmin + \frac{s^2 \signoise^2}{n^4} \\
    &\implies M^2 \Bigg(1 - c_1 - \frac{1}{n^4}\Bigg) \leq u^2 \signoise^2 + (s \signoise^2)^2 \Bigg((1 + c_1)\nmin + \frac{1}{n^4}\Bigg) \\
    &\implies M^2 \Bigg(1 - c_1 - \frac{1}{n^4}\Bigg) \leq u^2 \frac{\nmaj}{360000} + (s \signoise^2)^2 \Bigg((1 + c_1)\nmin + \frac{1}{n^4}\Bigg), \\
    &\text{From a bound on $\signoise^2$ in the parameter settings}.
  \end{align}
  Since $c_1 < 1/2000$, and $n \geq 2000$, setting $u = 1.3125, s \signoise^2 = 2.61$,  we get $M^2 \leq 10 n$.
\end{proof}

Now, we are ready to show that $\deltrain{\estmm, \gamma^2}$ is small.
\begin{lemma}
  \label{lem:deltrain}
  Under the parameter settings of Theorem~\ref{thm:over}, the following is true with high probability.
  \begin{align}
    \deltrain{\estmm, \frac{9}{10}} \leq 1/200,
  \end{align}
\end{lemma}
\begin{proof}

Applying \reflem{norm_bound_delta} to $\estmm$ by invoking the bounds on  $\is{\estmm}{i}$ (\reflem{supperbound}),
\begin{align}
  \| \estmm \|_2^2
    &\geq \frac{\gamma^4 (1 - c_1)}{\signoise^2}\deltrain{\estmm, \gamma^2}\nmaj - \frac{10}{\signoise^2 n^3}
\end{align}
with high probability. Putting this together with Lemma~\ref{lem:upperbound}, we have
\begin{align*}
  &\frac{\gamma^4 (1 - c_1)}{\signoise^2}\deltrain{\estmm, \gamma^2}\nmaj - \frac{10}{\signoise^2 n^3}
  \leq u^2 + s^2 \signoise^2 (1 + c_1) \nmin + \frac{s^2 \signoise^2}{n^4} \\
  &\implies \deltrain{\estmm, \gamma^2} \leq \underbrace{\frac{u^2 \signoise^2}{\gamma^4 \nmaj (1 - c_1)}}_{\text{Very small}} + \underbrace{\Bigg(\frac{(s \signoise^2)^2 (1 + c_1)}{\gamma^4 (1 - c_1)}\Bigg) \frac{\nmin}{\nmaj}}_{\approx 0.0042} + \underbrace{\frac{(s\signoise^2)^2}{n^4\nmaj}}_{\text{Very small}} + \underbrace{\frac{10}{\gamma^4 (1  - c_1)n^3}}_{\text{Very small}} \\
  &\implies \deltrain{\estmm, \frac{9}{10}} \leq 1/200, \text{w.h.p},
\end{align*}
where in the last step we substitute the constants $\gamma^2 = 9/10, u = 1.3125, s \signoise^2 = 2.61, \nmaj/\nmin \leq 1/2000$ and $\signoise^2 \leq \nmaj/360000$.
\end{proof}

\subsubsection{Concentration Inequalities}
\begin{lemma}
  \label{lem:allspubound}
 With probability $> 1 - 1/100$, if $\sigspu \leq \frac{1}{4 \sqrt{\log 100 n}}$,
  \begin{align}
    a - 1/5 \leq \ixspu{i}  \leq a + 1/5, ~\forall i = 1, \hdots n,
  \end{align}
  where $a$ is the spurious attribute.
\end{lemma}
This follows from standard subgaussian concentration and union bound over $n = \nmaj + \nmin$ points.

\begin{lemma}
\label{lem:norm}
  For a vector $z \in \R^N$ such that $z \in \sN(0, \sigma^2I)$,
  \begin{align}
     \BP ( | \| z \|^2 - \sigma^2 N | \geq \sigma^2 t) \leq 2 \exp \Big( \frac{- N t^2}{8} \Big).
  \end{align}
\end{lemma}

\begin{lemma}
  \label{lem:dotp1}
  For two vectors $z_i, z_j \in \R^N$ such that $z_i, z_j \sim \mathcal{N}(0, \sigma^2 I)$,
  by Hoeffding's inequality, we have
  \begin{align}
    \BP( |z_i \cdot z_j| \geq \sigma^2 t) \leq 2 \exp \Big( - \frac{t^2}{2 \| z_i \|^2 } \Big).
  \end{align}
\end{lemma}

\begin{corollary}
  \label{cor:dotp2}
  Combining Lemma~\ref{lem:norm} and Lemma~\ref{lem:dotp1}, we get
  \begin{align}
    \BP( |z_i\cdot z_j| \geq \sigma^2 t) \leq 2 \exp \Big( \frac{- N^3}{8} \Big) + 2 \exp \Big( - \frac{t^2}{8 N} \Big).
  \end{align}
\end{corollary}

\begin{lemma}
  \label{lem:alldotp}
For $N = \Omega(\text{poly}(n))$, with probability greater than $1 - 1/2000$,
\begin{align}
   |\ixnoi{i}\cdot\ixnoi{j}| \leq \frac{\signoise^2}{n^6}~~\forall \ixnoi{i},\ixnoi{j}.
  \end{align}
\end{lemma}
This follows from Corollary~\ref{cor:dotp2} and union bound over $n^2$ pairs of training points.
\begin{lemma}
  \label{lem:allnorm}
  For $N = \Omega(\text{poly}(n))$, with probability greater than $1 - 1/2000$,
  \begin{align}
    (1 - c_1) \sigma^2 \leq \| \ixnoi{i} \|^2 \leq (1 + c_1) \sigma^2, \forall i.
  \end{align}
  This follows from Lemma~\ref{lem:norm} and union bound over $n$ training points.
  In particular, we can set $c_1 < 1/2000$ for large enough $N$.
\end{lemma}

\subsubsection{Small $\deltrain{\estmm, \gamma^2}$ implies high worst-group error}
In the previous section, we proved that $\deltrain{\estmm, \gamma^2}$, the fraction of majority training samples that can have coefficient on the noise vectors greater than $\gamma^2/\signoise^2$ in the max margin separator $\estmm$ is bounded for suitable value of $\gamma$. We showed this using the fact that the norm of $\estmm$ is the smallest among all separators and the observation that the squared norm of a separator roughlty scales proportional the number of training points that have large coefficient along the noise vectors.

What does small $\deltrain{\estmm, \gamma^2}$ imply? We now show that the bound on $\deltrain{\estmm, \gamma^2}$ has an important consequence on the worst-group error $\roberr{\estmm}$;
low $\deltrain{\estmm, \gamma}$ would imply high worst-group error $\roberr{\estmm}$.
We show that there is a trade-off between the worst-group test error of a separator and the fraction of \emph{majority points} that it ``memorizes'' at training time.
If a model that has low worst-group test error must use the core feature and not the spurious feature, and to obtain zero training error such a model would memorize a potentially large fraction of majority and minority points. In contrast, if the model instead uses only the spurious feature, then the worst-group test error would be high, but it would memorize only a small fraction of majority examples at training time;
because we assume that the spurious feature is much less noisy than the core feature ($\sigcau \gg \sigspu$), much fewer majority examples would need to be memorized. To summarize,
\emph{a large $\spu{\hat{w}}$ would require smaller fraction of majority points to be memorized $\deltrain{\hat{w}, \gamma^2}$ but increase the worst-group test error $\roberr{\hat{w}}$}.
We formalize the above trade-off between the worst-group error and fraction of majority examples to be memorized in Proposition~\ref{prop:fracmaj-roberr}.
  \begin{restatable}{prop}{fracmajroberr}
    \label{prop:fracmaj-roberr}
    For the minimum norm separator $\estmm$, under the parameter settings of Theorem~\ref{thm:over}, with high probability,
    \begin{align}
      \roberr{\estmm} &\geq \Phi \Bigg ( \frac{-c_3 + \spu{\estmm} - \cau{\estmm}}{\sqrt{\cau{\estmm}{}^2 \sigcau^2 + \spu{\estmm}{}^2 \sigspu^2}} \Bigg) - c_4,
    \end{align}
    for some constants $c_3, c_4 < 1/1000$ and $\Phi$ the Gaussian CDF.

    For any separator $\hat{w}$ that spans the training points and satisfies
    \begin{align}
      {\is{\hat{w}}{i}}^2 \leq \frac{10 n}{\signoise^4},
    \end{align}
    under the parameter settings of Theorem~\ref{thm:over}, with high probability,
    \begin{align}
      \deltrain{\hat{w}, \gamma^2} &\geq \Phi \Bigg ( \frac{1 - (1 + c_1) \gamma^2 - c_5 - \spu{\hat{w}} - \cau{\hat{w}}}{\sqrt{\cau{\hat{w}}^2 \sigcau^2 + \spu{\hat{w}}^2 \sigspu^2}} \Bigg) - c_6,
    \end{align}
    for some constants $c_1 < 1/2000; c_5, c_6 < 1/1000$ and $\Phi$ the Gaussian CDF.
  \end{restatable}
  We prove \refprop{fracmaj-roberr} in \refsec{fracmaj-roberr}.

As mentioned before, we see that the spurious component weight $\spu{\estmm}$ has opposite effects on the two quantities;
$\roberr{\hat{w}}$ increases with increase $\spu{\hat{w}}$, but $\deltrain{\hat{w}, \gamma}$ decreases with increase in $\spu{\hat{w}}$. This dependence can be exploited to relate the two quantities to each other as follows.
  \begin{align}
  \label{eqn:tradeoff}
    \phiinv(\deltrain{\estmm, \gamma} + c_6) +  \phiinv(\roberr{\estmm} + c_4)
    \geq \frac{1 - c_3 - c_5 - (1 + c_1)\gamma^2  - 2\cau{\estmm} }{\sqrt{\cau{\hat{w}}^2 \sigcau^2 + \spu{\hat{w}}^2 \sigspu^2}}.
  \end{align}
  In other words, if the $\deltrain{\estmm, \gamma}$ is low, then $\roberr{\estmm}$ would need to be high.

\subsubsection{Worst-group error is high}
Recall from part 1 that $\deltrain{\estmm, \gamma} < 1/200$ for appropriate choice of $\gamma$, and from part 2 the trade-off between $\deltrain{\estmm, \gamma}$ and $\roberr{\estmm}$ (Equation~\refeqn{tradeoff}). As a final step, we need to bound the quantities on the RHS of Equation~\refeqn{tradeoff}. All the constants are small, and $\gamma^2 = 9/10, \deltrain{\estmm, 9/10} \leq 1/200 $ (Lemma~\ref{lem:deltrain}) which allows us to write
\begin{align}
      \phiinv(0.006) + \phiinv(\roberr{\estmm} + c_4) &\geq \frac{-2 \cau{\estmm}}{\sqrt{\cau{\estmm}{}^2 \sigcau^2 + \spu{\estmm}{}^2 \sigspu^2}} \geq \frac{-2}{\sigcau} \\
      \implies \phiinv(\roberr{\estmm} + c_4) &\geq 0.512 \\
     \implies  \roberr{\estmm} &\geq 0.67
    \end{align}
    We have hence proved that the minimum-norm separator $\estmm$ incurs high worst-group error with high probability under the specified conditions.

\subsubsection{Proof of \refprop{fracmaj-roberr}}
\label{sec:fracmaj-roberr}
\fracmajroberr*
\begin{proof}
We derive the two bounds below.

\underlinepara{Worst-group test error}
~\\
  We bound the expected worst-group error $\roberr{\estmm}$, which is the expected worst-group loss over the data distribution. Below, we lower bound the worst-group error $\roberr{\estmm}$ by bounding the error on a particular group: minority positive points which have label $y=1$ and spurious attribute $a=-1$.
  The test error is the probability that a test example $x$ from this group gets misclassified, i.e. $\estmm \cdot x < 0$.
  \begin{align}
    \roberr{\estmm}
    &\ge \BP \left(\estmm \cdot x < 0 \mid y=1, a=-1 \right)\\
    &= \BP \left(\cau{\estmm}\xcau  + \spu{\estmm}\xspu + \noise{\estmm} \cdot \xnoise < 0 \mid y=1, a=-1 \right)\\
    &= \BP \Big( \cau{\estmm}(1 + \sigcau z_1) + \spu{\estmm}(-1 + \sigspu z_2) + \noise{\estmm} \cdot \xnoise < 0 \Big)
  \end{align}
  In the last step, we rewrite for convenience $\xcau = y + \sigcau z_1$ and $\xspu = a + \sigspu z_2$, where $z_1, z_2 \sim \sN(0, 1)$.

  We use the properties of high-dimensional Gaussian random vectors to bound the quantity $\noise{\estmm} \cdot \xnoise$.
  Recall that $\noise{\estmm}$ can be written as
  \begin{align}
    \noise{\estmm} = \sum \limits_{i \in \gmaj, \gmin} \is{\estmm}{i} \ixnoi{i}.
  \end{align}
  From Lemma~\ref{lem:supperbound}, we know that $\max \limits_i {\is{\estmm}{i}}^2 < \frac{10n}{\signoise^4}$.
  This, along with Lemma~\ref{lem:dotp1} gives $| \xnoise \cdot \noise{\estmm} | \leq c_3$ with probability $1 - c_4$ for some small constants $c_3, c_4 < 1/1000$.
  Let $B$ denote the event that this high probability event where the dot product $| \xnoise \cdot \noise{\estmm} | \leq c_3$. Using the fact that $\BP(A) \geq \BP(A \mid B) - \BP(\neg B)$ which follows from simple algebra, we have
  \begin{align}
    \roberr{\estmm} &\geq \BP \Big( \cau{\estmm}(1 + \sigcau z_1) + \spu{\estmm}(-1 + \sigspu z_2) + \noise{\estmm} \cdot \xnoise < 0 \Big) \\
    &\ge \BP \Big( \cau{\estmm}(1 + \sigcau z_1) + \spu{\estmm}(1 - \sigspu z_2) < -c_3 \Big) - c_4\\
    &= \BP \Big( \cau{\estmm}\sigcau z_1 + \spu{\estmm}\sigspu z_2 < -c_3 + \spu{\estmm} - \cau{\estmm} \Big) - c_4\\
    &= \Phi \Bigg ( \frac{-c_3 + \spu{\estmm} - \cau{\estmm}}{\sqrt{\cau{\estmm}{}^2 \sigcau^2 + \spu{\estmm}{}^2 \sigspu^2}} \Bigg) - c_4.
  \end{align}
  From the expression above, we see that $\roberr{\estmm}$ increases as the spurious component $\spu{\estmm}$ increases.
  This is because in the minority group, the spurious feature is negatively correlated with the label.

\underlinepara{Fraction of memorized training examples in majority groups}
~\\
We now compute a lower bound on $\deltrain{\estmm, \gamma^2}$, which is the number of majority points (where $a = y$) that are ``memorized.'' Intuitively, we want to show that the fraction depends on $\spu{\hat{w}} - \cau{\hat{w}}$. The more the core feature is used relative to the spurious feature, the larger fraction of points need to be memorized because the core feature is more noisy.

First, consider a separator $\hat{w}$ with some core and spurious components $\cau{\hat{w}}$ and $\spu{\hat{w}}$.
Recall that $\noise{\hat{w}} = \sum \limits_i \is{\hat{w}}{i} \ixnoi{i}$ and $\iy{i}(\hat{w}\cdot\ix{i})\ge1$ by the definition of separators. For a given $\cau{\hat{w}}$ and $\spu{\hat{w}}$, we want to bound the fraction of majority points ($a = y$) which can have $\is{\hat{w}}{i} < \frac{\gamma^2}{\signoise^2}$. We focus only on separators with bounded memorization, i.e. those that satisfy ${\is{\hat{w}}{i}}^2 \leq \frac{10 n}{\signoise^4}$. Note that from Lemma~\ref{lem:supperbound}, w.h.p., the mininum-norm separator $\estmm$ satifies this condition.

We bound the above by bounding a related quantity: the fraction of points that are memorized in the training distribution in expectation. We then use concentration to relate it to the fraction of the training set.

Formally, we have fixed quantities $\cau{\hat{w}}$ and $\spu{\hat{w}}$. The training set is generated as per the usual data generating distribution. As before, we are interested in separators on the training set. For any majority training point, the coefficient $\is{\hat{w}}{i}$ in a separator is a random variable. Since training point $i$ is separated, we have
\begin{align*}
  \cau{\hat{w}}(1 + \sigcau z_1) + \spu{\hat{w}}(1 + \sigspu z_2) + \Big(\sum \limits_{i} \is{\hat{w}}{i} \ixnoi{i} \Big)^\top \ixnoi{i} \geq 1.
\end{align*}
From Lemma~\ref{lem:alldotp}, Lemma~\ref{lem:norm}, and the condition on $\is{\hat{w}}{i}$, this implies with high probability that
\begin{align*}
  \cau{\hat{w}}(1 + \sigcau z_1) + \spu{\hat{w}}(1 + \sigspu z_2)
\geq 1 - (1 + c_1)\signoise^2 \is{\hat{w}}{i} - c_5,
\end{align*}
for some constant $c_5 < 1/1000$.
Conditioning on the high probability event just as before ($\BP(A) \leq \BP(A \mid B) + \BP(\neg B)$), we get
\begin{align}
  \BP (\is{\hat{w}}{i} \leq \frac{\gamma^2}{\signoise^2}) &\leq \BP \Big( \cau{\hat{w}}\sigcau z_1 + \spu{\hat{w}}\sigspu z_2 \leq - 1 + (1 + c_1) \gamma^2 + c_5 + \cau{\hat{w}} + \spu{\hat{w}} \Big) + \delta \\
  &= \Phi \Bigg( \frac{-1 + (1 + c_1) \gamma^2 + c_5 + \spu{\hat{w}} + \cau{\hat{w}}}{\sqrt{\cau{\hat{w}}^2 \sigcau^2 + \spu{\hat{w}}^2 \sigspu^2}} \Bigg) + \delta \\
  \implies \BP (\is{\hat{w}}{i} \geq \frac{\gamma^2}{\signoise^2}) &\geq \Phi \Bigg ( \frac{1 - (1 + c_1) \gamma^2 - c_5 - \spu{\hat{w}} - \cau{\hat{w}}}{\sqrt{\cau{\hat{w}}^2 \sigcau^2 + \spu{\hat{w}}^2 \sigspu^2}} \Bigg) - \delta,
\end{align}
for some $\delta < 1/2000$.
Finally, we connect to $\deltrain{\hat{w}}(\gamma^2)$ which is the finite sample version of the quantity $\BP (\is{\hat{w}}{i} \leq \frac{\gamma^2}{\signoise^2})$. By DKW, we know that the empirical CDF converges to the population CDF.
Under the conditions of Theorem~\ref{thm:over}, which lower bounds the number of majority elements, we have with high probability,
\begin{align}
  \deltrain{\hat{w}}(\gamma^2) \geq \Phi \Bigg ( \frac{1 - (1 + c_1) \gamma^2 - c_5 - \spu{\hat{w}} - \cau{\hat{w}}}{\sqrt{\cau{\hat{w}}^2 \sigcau^2 + \spu{\hat{w}}^2 \sigspu^2}} \Bigg) - c_6,
\end{align}
for constants $c_5, c_6 < 1/1000$.

\end{proof}
\end{proof}

\subsubsection{Proof of \refprop{spu-cost}}
\label{sec:app-spu-cost}
\spucost*
\begin{proof}
  The proposition follows directly from Lemma~\ref{lem:upperbound}.
  \begin{align*}
    \| \memmin \|_2^2 &\leq u^2 + s^2 \signoise^2 (1 + c_1) \nmin + \frac{s^2 \signoise^2}{n^4} \\
                      &\leq u^2 + s^2 \signoise^2 (2 + c_1) \nmin.
  \end{align*}
  The constant $\gamma_1 = u = 1.3125 $ and $\gamma_2 = s \signoise^2(2 + c_1) = 2.61(2 + c_1)$ for $c_1<1/2000$.
\end{proof}

\subsubsection{Proof of \refprop{core-cost}}
\label{sec:app-core-cost}
\corecost*
\begin{proof}
To bound the norm for all $\memmaj\in\Memmaj$, we provide a lower bound on the norm of the minimum-norm separator in the set $\Memmaj$:
\begin{align}
\memmajmm \eqdef \arg\min_{w\in\Memmaj}\|w\|^2.
\end{align}
We bound the $\|\memmajmm\|$ in two steps:
\begin{enumerate}
\item We first provide a lower bound for $\|\memmajmm\|$ in terms of the fraction of training points memorized  $\deltrainall{\memmajmm,\gamma^2}$ (defined formally below) in \refcor{norm_bound_deltrain_memmaj}.
\item We then provide a lower bound for $\deltrainall{\memmajmm,\gamma^2}$ in \refcor{deltrain_bound_memmaj}.
\end{enumerate}

We first formally define $\deltrainall{\hat{w},\gamma^2}$.
\begin{definition}
\label{def:deltrainall}
For a separator $\hat{w}$ on training data $\{(\ix{i}, \iy{i})\}_{i=1}^n$,
let $\deltrainall{\hat{w}, \gamma^2}$ be the fraction of training examples that $\hat{w}$ $\gamma$-memorizes:
\begin{align}
\deltrainall{\hat{w}, \gamma^2} \eqdef \frac{1}{n} \sum_{i=1}^n \mathbb{I}\left[\left|\is{\hat{w}}{i}\right| > \frac{\gamma^2}{\signoise^2}\right]
\end{align}
\end{definition}

\underlinepara{Bounding $\|\memmajmm\|$ by $\deltrainall{\memmajmm,\gamma^2}$}
\begin{lemma}
\label{lem:norm_bound_deltrainall}
For a separator $\hat{w}$ with bounded $\is{\hat{w}}{i}^2 \le \frac{10n}{\signoise^2}$ for all $i=1,\dots,n$, its norm can be bounded with high probability as
\begin{align}
 \| \hat{w} \|_2^2 &\geq \frac{\gamma^4 (1 - c_1)}{\signoise^2}\deltrainall{\hat{w}, \gamma^2}n - \frac{10}{\signoise^2 n^3}
\end{align}
\end{lemma}
\begin{proof}
Similarly to the proof of \reflem{norm_bound_delta}, the result follows bounded norms (\reflem{allnorm}), bounded dot products (\reflem{alldotp}), and the definition of $\deltrainall{\hat{w},\gamma^2}$ (\refdef{deltrainall}).
\begin{align}
  \| \hat{w} \|_2^2 &\geq \sum \limits_{i \in \gmaj} {\is{\hat{w}}{i}}^2 \| \ixnoi{i} \|_2^2 + \sum \limits_{j\neq k} \is{\hat{w}}{j} \is{\hat{w}}{k} \ixnoi{j} \cdot \ixnoi{k} \\
  &\geq \underbrace{\Big(\frac{\gamma^4 (1 - c_1) }{\signoise^2}\Big)\deltrainall{\hat{w}, \gamma^2} n}_{\text{Choosing only points with $\is{\hat{w}}{i} \geq \gamma^2/\signoise^2$}} - \underbrace{\frac{M^2}{\signoise^2 n^4}}_{\text{$\max \is{\hat{w}}{i} = M/\signoise^2$}}, ~\text{w.h.p.} \\
    &\geq \frac{\gamma^4 (1 - c_1)}{\signoise^2}\deltrainall{\hat{w}, \gamma^2}n - \frac{10}{\signoise^2 n^3}
\end{align}
\end{proof}

\begin{corollary}
 With high probability,
\label{cor:norm_bound_deltrain_memmaj}
\begin{align}
  \| \memmajmm\|_2^2
    &\geq \frac{\gamma^4 (1 - c_1)}{\signoise^2}\deltrain{\memmajmm, \gamma^2}\nmaj - \frac{10}{\signoise^2 n^3}
\end{align}
\end{corollary}
\begin{proof}
The result follows from applying \reflem{norm_bound_deltrainall} to $\memmajmm$, invoking the bounds on any individual component $\is{\memmajmm}{i}$ obtained below in \reflem{supperbound_memmaj}.
\end{proof}

Below, we bound $\is{\memmajmm}{i}$, where $\is{\memmajmm}{i}$ is the component of training point $i$ to the classifier $\memmajmm$ via the representer theorem.
\begin{lemma}
\label{lem:supperbound_memmaj}
With high probability, $i=1,\dots,n$, $\is{\memmajmm}{i}$ can be bounded as follows.
  \begin{align}
    {\is{\memmajmm}{i}}^2 \leq \frac{10n}{\signoise^4}.
  \end{align}
\end{lemma}
\begin{proof}
As a first step, we upper bound the norm of $\memmajmm$ by the norm of another separator $\memall \in \Memmaj$, using the fact that $\memmajmm$ is the minimum-norm separator in $\Memmaj$.
In particular, we construct a separator $\memall \in \Memmaj$ that ``memorizes'' all training points, of the following form:
\begin{align*}
  \cau{\memall} &= 0\\
  \spu{\memall} &= 0\\
  \is{\memall}{i}&=\iy{i} \alpha \text{ for all } i=1,\dots,n.
\end{align*}
This is analogous to the construction of $\memmin \in \Memmin$ (Lemma~\ref{lem:upperbound}), and similar calculations can be used to obtain a suitable value $\alpha$ to ensure that $\memall$ is a separator with high probability.
We provide it below for completeness.
We show that the following condition is sufficient to satisfy the margin constraints $\iy{i}\memall \cdot \ix{i}\ge 1$ for all $i=1,\dots,n$ with high probability:
\begin{align}
\alpha\signoise^2 \ge \frac{1}{1-c_1-1/n^5}.
\end{align}
for $c_1<1/2000$.
We obtain the above condition by applying \reflem{alldotp} and \reflem{allnorm} to the margin condition.
\begin{align}
  &~~\memall\cdot\ix{i} \ge 1\\
  &\implies \alpha \|\ixnoi{i}\|^2 - \alpha \sum_{j\neq i} \left|\ixnoi{i}\cdot\ixnoi{j}\right|\ge 1\\
  &\implies \alpha\signoise^2 (1-c_1)  - \frac{\alpha\signoise^2}{n^5}\ge 1 ~ \text{ with high probability}
\end{align}
Thus, we can construct $\memall$ by setting some constant $\alpha\signoise^2\le2$.

Now that we have constructed $\memall$, we can bound the norm of the minimum norm separator $\memmajmm$ by the norm of $\memall$. The following is true with high probability,
\begin{align}
\|\memmajmm\|^2
&\le \|\noise{\memall}\|^2\\
&= \sum_{i=1}^n \alpha^2\|\ixnoi{i}\|^2 + \sum_{i\neq j} \alpha^2 \ixnoi{i}\cdot\ixnoi{j}\\
&\le \alpha^2\signoise^2 (1+c_1) n + \frac{\alpha^2 \signoise^2}{n^4}
\label{eqn:memall_norm}
\end{align}

Finally, we bound $\is{\memmajmm}{i}$ for all $i$ by bounding $\max \limits_i \is{\memmajmm}{i} = \frac{M}{\signoise^2}$.
As we showed in the proof of \reflem{supperbound}, following is true with high probability:
  \begin{align}
    \| \memmajmm\|_2^2
    &\geq \frac{M^2 (1 - c_1)}{\signoise^2}  - \frac{M^2}{\signoise^2 n^4}.
  \end{align}
Combined with the upper bound on $\| \memmajmm\|_2^2$ (Equation~\refeqn{memall_norm}), we have
  \begin{align}
   & \frac{M^2 (1 - c_1)}{\signoise^2}  - \frac{M^2}{\signoise^2 n^4} \leq \|\memmajmm\|\leq \alpha^2 \signoise^2 (1 + c_1) n + \frac{\alpha^2 \signoise^2}{n^4} \\
    &\implies M^2 \Bigg(1 - c_1 - \frac{1}{n^4}\Bigg) \leq (\alpha \signoise^2)^2 \Bigg((1 + c_1)n + \frac{1}{n^4}\Bigg).
  \end{align}
  Since $c_1 < 1/2000$, and $n \geq 2000$, setting $\alpha \signoise^2 = 2$ yields $M^2 \leq 10 n$ with high probability.
\end{proof}

\underlinepara{Bounding $\deltrainall{\memmajmm,\gamma^2}$}
\begin{corollary}
\label{cor:deltrain_bound_memmaj}
    Under the parameter settings of Theorem~\ref{thm:over}, with high probability,
    \begin{align}
      \deltrainall{\memmajmm, \gamma^2} &\geq \Phi \Bigg ( \frac{1 - (1 + c_1) \gamma^2 - c_5 - \cau{\memmajmm}}{\left|\cau{\memmajmm}\sigcau\right|} \Bigg) - c_6,
    \end{align}
    for some constants $c_1 < 1/2000; c_5, c_6 < 1/1000$ where $\Phi$ is the Gaussian CDF.
\end{corollary}
\begin{proof}
  The result follows from applying \refprop{fracmaj-roberr} (which computes a bound on the majority fraction of points that is $\gamma-$memorized) to $\memmajmm$, invoking \reflem{supperbound_memmaj}, and plugging in $\spu{\memmajmm}=0$. Note that when $\spu{\memmajmm} = 0$, $\deltrainall{\memmajmm,\gamma^2} = \deltrain{\memmajmm,\gamma^2}$.
\end{proof}

Finally, the above bound on $\deltrainall{\memmajmm,\gamma^2}$ translates to a bound on the norm $\| \memmajmm \|$ via simple algebra. For $\gamma$ that satisfies $1-(1 + c_1) \gamma^2 - c_5>0$:
\begin{align}
\deltrainall{\memmajmm, \gamma^2}
&\ge \Phi\left(\frac{-1}{\sigcau} + \frac{1-(1 + c_1) \gamma^2 - c_5}{\left|\cau{\memmajmm}\sigcau\right|}\right) - c_6\\
&\ge \Phi\left(\frac{-1}{\sigcau}\right) - c_6.
\end{align}

Plugging the above lower bound into the bound on $\|\memmajmm\|$ from \refcor{norm_bound_deltrain_memmaj}, we have
\begin{align}
\| \memmajmm\|_2^2
&\geq \frac{\gamma^4 (1 - c_1)}{\signoise^2}\deltrainall{\memmajmm, \gamma^2}\nmaj - \frac{10}{\signoise^2 n^3}\\
&\ge \frac{n}{\signoise^2}\left(\Phi\left(\frac{-1}{\sigcau}\right) - c_6\right)\gamma^4 (1 - c_1) - \frac{10}{\signoise^2 n^3}\\
&\ge \frac{n}{\signoise^2}\underbrace{\left[\left(\Phi\left(\frac{-1}{\sigcau}\right) - c_6\right)\gamma^4 (1 - c_1) - c_7\right]}_{\text{set to }\gamma_3}
\end{align}
for some $c_7<1/1000$.
\end{proof}

\subsection{Underparameterized regime}
\label{sec:app-under}

So far, we have studied the overparameterized regime for the data distribution described in Section~\ref{sec:results_mechanism}. In the overparameterized setting, where the dimension of noise features $N$ is very large, logistic regression (both ERM and reweighted) leads to max-margin classifiers. We showed that for some setting of parameters $\nmaj, \nmin, \sigspu, \sigcau$, the robust error of such max-margin classifiers can be $> 2/3$, worse than random guessing.
How does the same reweighted logistic regression perform in the underparameterized regime?
We focus on the setting where $N = 0$.
In this setting, the data is two-dimensional, and w.h.p., the training data is not linearly separable unless $\sigcau = 0$.
Consequently, the learned model $\estrw \R^2$ that minimizes the reweighted training loss is not generally a max-margin separator.

For intuition, consider the following two sets of models, which are analogous to what we considered in Equation~\ref{eq:sep} in the main text for the overparameterized regime:
\begin{align}
  \Memmin \eqdef \{ w \in \R^{2} ~~\text{such that}~~\cau{w} = 0\} \nonumber\\
  \Memmaj \eqdef \{ w \in \R^{2} ~~\text{such that}~~\spu{w} = 0\}.
\end{align}
The first set $\Memmin$ comprises models that use the spurious feature but not the core feature, and the second set $\Memmaj$ comprises models that use the core feature but not the spurious feature. Models in $\Memmin$ that exclusively use $\xspu$ will have high training loss on the minorities since the minority points cannot be memorized. Due to upweighting the minorities, these models will have high reweighted training loss. On the other hand, models in $\Memmaj$ exclusively use the core features that are informative for the label $y$ across all groups. Hence they obtain reasonable loss across all groups and have smaller reweighted training loss than models in $\Memmin$.

We will show in this section that the population minimizer of the reweighted loss is indeed in $\Memmaj$ and bound the asymptotic variance of the reweighted estimator, leading to the final result in \refthm{main}.
Our approach is to study the asypmtotic behavior of the reweighted estimator when the number of data points $n \gg d$.

\paragraph{Data distribution.} We first recap the data generating distribution (described in Section~\ref{sec:results_mechanism}).
$x = [\xcau, \xspu]$ where,
\begin{align*}
  \xcau \mid y \sim \sN(y,\sigcau^2), ~~&~~ \xspu \mid a \sim \sN(a,\sigspu^2),
\end{align*}
For $\pmajmath$ fraction of points, we have $a = y$ (majority points) and for $1 - \pmajmath$ fraction of points, we have $a = -y$ (minority points).

\paragraph{Reweighted logistic loss.}
Let $\pmajmath$ be the fraction of the majority group points and $(1 - \pmajmath)$ be the fraction of minority points. In order to use standard results from the asymptotics of M-estimators, we rewrite the reweighted estimator (defined in Section~\ref{sec:setup}) as the minimizer of the following loss over $n$ training points $[x_i, y_i]_{i=1}^n$.
\begin{align}
  \label{eq:losses}
  \estrw &= \arg\min \frac{1}{n} \sum \limits_{i=1}^n \ellrw(x_i, y_i, w) \\
  \ellrw(x, y, w) &= \frac{-1}{\pmajmath} \log \Bigg(\frac{1}{1 + \exp(-yw^\top x)}\Bigg), ~\text{For $(x, y)$ from majority group} \\
  \ellrw(x, y, w) &= \frac{-1}{1 - \pmajmath} \log \Bigg(\frac{1}{1 + \exp(-yw^\top x)}\Bigg), ~\text{For $(x, y)$ from minority group}.
\end{align}
We follow the standard steps of asymptotic analysis where we:
\begin{enumerate}
\item Compute the population minimizer $w^\star$ that satisfies $\nabla \Lrw(w^\star) = 0$, where $\Lrw(w^\star) = \E [ \ellrw(x, y, w^\star)]$.
\item Bound the asymptotic variance $\nabla^2 \Lrw(w^\star)^{-1} \cov[ \nabla \ellrw(x, y, w^\star)] \nabla^2 \Lrw(w^\star)^{-1}$.
\end{enumerate}

\begin{restatable}[]{prop}{popminimizer}
  \label{prop:popminimizer}
  For the data distribution under study, the population minimizer $w^\star$ that satisfies $\nabla \Lrw(w^\star) = 0$ is the following.
 \begin{align}
   w^\star = \left[ \frac{2}{\sigcau^2}, 0\right].
 \end{align}
\end{restatable}
This is a very important property in the underparameterized regime: the population minimizer has the best possible worst-group error by only using the core feature and not the spurious feature.
\begin{restatable}[]{prop}{finalasymptotic}
\label{prop:finalasymptotic}
 The asymptotic distribution of the reweighted logistic regression estimator is as follows.
 \begin{align}
    \sqrt{n}(\hat{w} - w^\star) &\rightarrow^d \sN(0, V), \\
    V \preceq \diag \Bigg( \frac{16\exp \Big(\frac{8}{(\sigcau^2 + 8)\sigcau^2} \Big)(\sigcau^2 + 1)(1 + 8/\sigcau^2)^3}{\pmajmath(1 - \pmajmath)(\sigcau^2 + 9)^2}, &
    \frac{16\exp \Big(\frac{8}{(\sigcau^2 + 8)\sigcau^2} \Big)(1 + 8/\sigcau^2)}{\pmajmath(1 - \pmajmath)(\sigspu^2 + 1)} \Bigg).
    \label{eq:asyvariance}
 \end{align}
 For $\sigcau \geq 1$, we have
 \begin{align}
        V &\preceq \diag \Bigg( \frac{C_1}{\pmajmath(1 - \pmajmath)},
    \frac{C_2}{\pmajmath(1 - \pmajmath)} \Bigg),
 \end{align}
 for some constants $C_1, C_2$.
\end{restatable}
We see that the asymptotic variance increases as $\pmajmath$ increases. This is expected because the reweighted estimator upweights the minority points by inverse of group size. As these weights increase, the variance also increases.
However, as we noted before, since the population minimizer has small worst-group error, for large enough training set size, we get small worst-group error since the asymptotic variance is finite (for fixed $\pmajmath$) and the estimator approaches the population minimizer.

We now prove \refthm{main} for the underparameterized regime, restated as \refthm{under} below.
\label{sec:app-under-logistic}
\begin{restatable}[]{thm}{under}
  \label{thm:under}
  In the underparameterized regime with $N=0$, for $\pmajmath=\bigl(1-\frac{1}{2001}\bigr)$, $\sigcau^2 = 1$, and $\sigspu^2 = 0$, in the asymptotic regime with $\nmaj, \nmin \rightarrow \infty$, we have
  \begin{align}
    \roberr{\estrw} < 1/4.
  \end{align}
 \end{restatable}

\begin{proof}
  We now put the two Propositions~\ref{prop:finalasymptotic} and~\ref{prop:popminimizer} together. We have $\cau{\estrw} \geq 2 - \epsilon_1$ and $|\spu{\estrw}| \leq \epsilon_2$ for $\epsilon_1, \epsilon_2 < 1/10$, i.e the estimator is very close to the population minimizer. This follows from setting $\sigcau, \sigspu, \pmajmath = \frac{\nmaj}{\nmaj + \nmin}$ to their corresponding values and setting $n = \nmaj + \nmin$ to be large enough.
In order to compute the worst-group error, WLOG consider points with label $y = 1$ (labels are balanced in the population). For a point from the majority group, the probability of misclassification is as follows.
\begin{align}
\Pr[ \cau{\estrw}\xcau + \spu{\estrw}\xspu \geq 0 ] = \Pr[ z \geq \frac{\cau{\estrw}+ \spu{\estrw}}{\sigcau^2 \cau{\estrw}{}^2 + \sigspu^2 \spu{\estrw}{}^2}],
\end{align}
where $z \sim \sN(0, 1)$.

Similarly, for the minority group, the probability of misclassification is
\begin{align}
\Pr[ z \geq \frac{\cau{\estrw} - \spu{\estrw}}{\sigcau^2 \cau{\estrw}{}^2 + \sigspu^2 \spu{\estrw}{}^2}], ~\text{where $z \sim \sN(0, 1)$}.
\end{align}
Therefore, the worst-group error of $\estrw$ can be bounded as.
\begin{align}
\roberr{\estrw} \leq 1 - \Phi \Bigg( \frac{\cau{\estrw} - |\spu{\estrw}|}{\sigcau^2 \cau{\estrw}{}^2 + \sigspu^2 \spu{\estrw}{}^2 }\Bigg),
\end{align}
where $\Phi$ is the Gaussian CDF.
Substituting $\sigcau = 1, \sigspu = 0, \cau{\estrw} \geq 2 - \epsilon_1, |\spu{\estrw}| \leq \epsilon_2$ gives the required result that $\roberr{\estrw} < 1/4$. In contrast, in the overparameterized regime where $N \gg n$, even for very large $n$, the reweighted estimator has high worst-group error, as shown in Theorem~\ref{thm:main}.
\end{proof}

\subsubsection{Complete proofs}
\label{sec:complete-proofs}
We now provide the proofs for Proposition~\ref{prop:popminimizer} and Proposition~\ref{prop:finalasymptotic} which mostly follow from straightforward algebra.
\popminimizer*
\begin{proof}
  For convenience, we compute expectations over the majority and minority groups separately and express the population loss $\Lrw$ as the weighted sum of the two terms. Recall that we denote $x = [\xcau, \xspu]$.
  \begin{align}
    \label{eq:poplosses}
      \Lrw(w) &= \pmajmath \Lrwmaj + (1 - \pmajmath) \Lrwmin \\
      \Lrwmaj(w) &= \E_y \E_{\xcau \sim \sN(y, \sigcau^2)}\E_{\xspu \sim \sN(y, \sigspu^2)} [\ellrw(x, y, w)].\\
      \Lrwmin(w) &= \E_y \E_{\xcau \sim \sN(y, \sigcau^2)}\E_{\xspu \sim \sN(-y, \sigspu^2)} [\ellrw(x, y, w)].
    \end{align}
    We use the following expression for computing the population gradient.
    \begin{align}
      \nabla \log \Bigg(\frac{1}{1 + \exp(-yw^\top x)}\Bigg) &= \Bigg(\frac{-y \exp(- yw^\top x)}{1 + \exp(-yw^\top x)}\Bigg)x.
    \end{align}
    Combining the definition of the reweighted loss and population losses (Equation~\ref{eq:losses} and Equation~\ref{eq:poplosses}) with the gradient expression above gives the following.
\begin{align}
    \nabla \Lrwmaj(w)  &=  \E_y \E_{\xcau \sim \sN(y, \sigcau^2)}\E_{\xspu \sim \sN(y, \sigspu^2)} \left[\frac{1}{\pmajmath} \Bigg(\frac{-y \exp(- yw^\top x)}{1 + \exp(-yw^\top x)}\Bigg)x \right]. \\
    \nabla \Lrwmin(w)  &=  \E_y \E_{\xcau \sim \sN(y, \sigcau^2)}\E_{\xspu \sim \sN(-y, \sigspu^2)} \left[ \frac{1}{1 - \pmajmath}\Bigg(\frac{-y \exp(- yw^\top x)}{1 + \exp(-yw^\top x)}\Bigg)x \right].
  \end{align}
Now we compute $\nabla \Lrw(w^\star) = \pmajmath \nabla \Lrwmaj(w^\star) + (1 - \pmajmath) \nabla \Lrwmin(w^\star)$. First we compute wrt the spurious attribute $\nabla_{\textsf{spu}} \Lrw(w^\star)$.
For convenience, let $c = \frac{2}{\sigcau^2}$.
\begin{align*}
  \nabla_\textsf{spu}\Lrwmaj(w^\star)  &=  \E_y \E_{\xcau \sim \sN(y, \sigcau^2)}\E_{\xspu \sim \sN(y, \sigspu^2)} \left[\frac{1}{\pmajmath} \Bigg(\frac{-y \exp(-y c \xcau)}{1 + \exp(-y c \xcau)}\Bigg)\xspu\right] \\
  &= \frac{1}{2} \E_{\xcau \sim \sN(1, \sigcau^2)}\E_{\xspu \sim \sN(1, \sigspu^2)} \left[\frac{1}{\pmajmath} \Bigg(\frac{-\exp(- c \xcau)}{1 + \exp(- c \xcau)}\Bigg)\xspu\right] \\
  &+ \frac{1}{2} \E_{\xcau \sim \sN(-1, \sigcau^2)}\E_{\xspu \sim \sN(-1, \sigspu^2)} \left[\frac{1}{\pmajmath} \Bigg(\frac{ \exp(c \xcau)}{1 + \exp(c \xcau)}\Bigg)\xspu\right]\\
  &= \frac{1}{2} \E_{\xcau \sim \sN(1, \sigcau^2)}\left[\frac{1}{\pmajmath} \Bigg(\frac{-\exp(- c \xcau)}{1 + \exp(- c \xcau)}\Bigg)\right] - \frac{1}{2} \E_{\xcau \sim \sN(-1, \sigcau^2)}\left[\frac{1}{\pmajmath} \Bigg(\frac{ \exp(c \xcau)}{1 + \exp(c \xcau)}\Bigg)\right]\\
  &= \frac{1}{2} \E_{\xcau \sim \sN(1, \sigcau^2)}\left[\frac{1}{\pmajmath} \Bigg(\frac{-\exp(- c \xcau)}{1 + \exp(- c \xcau)}\Bigg)\right] - \underbrace{\frac{1}{2} \E_{\xcau \sim \sN(1, \sigcau^2)}\left[\frac{1}{\pmajmath} \Bigg(\frac{ \exp(-c \xcau)}{1 + \exp(-c \xcau)}\Bigg)\right]}_\text{Replacing $\xcau \sim \sN(-1, \sigcau^2)$ with $-\xcau \sim \sN(1, \sigcau^2)$} \\
  & =\E_{\xcau \sim \sN(1, \sigcau^2)}\left[\frac{1}{\pmajmath} \Bigg(\frac{-\exp(- c \xcau)}{1 + \exp(- c \xcau)}\Bigg)\right] \\
  \nabla_\textsf{spu}\Lrwmin(w^\star)  &=  \E_y \E_{\xcau \sim \sN(y, \sigcau^2)}\E_{\xspu \sim \sN(-y, \sigspu^2)} \left[\frac{1}{1 - \pmajmath} \Bigg(\frac{-y \exp(-y c \xcau)}{1 + \exp(-y c \xcau)}\Bigg)\xspu\right] \\
  &= \frac{1}{2} \E_{\xcau \sim \sN(1, \sigcau^2)}\left[\frac{1}{\pmajmath} \Bigg(\frac{\exp(- c \xcau)}{1 + \exp(- c \xcau)}\Bigg)\right] + \frac{1}{2} \E_{\xcau \sim \sN(-1, \sigcau^2)}\left[\frac{1}{\pmajmath} \Bigg(\frac{ \exp(c \xcau)}{1 + \exp(c \xcau)}\Bigg)\right] \\
  &= \E_{\xcau \sim \sN(1, \sigcau^2)}\left[\frac{1}{1 - \pmajmath} \Bigg(\frac{\exp(- c \xcau)}{1 + \exp(- c \xcau)}\Bigg)\right]
\end{align*}
Now we take the weighted combination of $\nabla_\textsf{spu}\Lrwmaj(w^\star)$ and $ \nabla_\textsf{spu}\Lrwmin(w^\star)$, based on the fraction of the majority and minority samples in the population, which makes the two terms cancel out.
\begin{align}
  \nabla_\textsf{spu} \Lrw &= \pmajmath \nabla_\textsf{spu}\Lrwmaj(w^\star) + (1 - \pmajmath) \nabla_\textsf{spu}\Lrwmin(w^\star) = 0.
\end{align}

Now we compute $\nabla_{\textsf{core}} \Lrw(w^\star)$.
\begin{align*}
  \nabla_\textsf{core}\Lrwmaj(w^\star)  &=  \E_y \E_{\xcau \sim \sN(y, \sigcau^2)}\E_{\xspu \sim \sN(y, \sigspu^2)} \left[\frac{1}{\pmajmath} \Bigg(\frac{-y \exp(-y c \xcau)}{1 + \exp(-y c \xcau)}\Bigg)\xcau\right] \\
  &= \frac{1}{2} \E_{\xcau \sim \sN(1, \sigcau^2)}\left[\frac{1}{\pmajmath} \Bigg(\frac{-\exp(- c \xcau)}{1 + \exp(- c \xcau)}\Bigg)\xcau \right] + \frac{1}{2} \E_{\xcau \sim \sN(-1, \sigcau^2)}\left[\frac{1}{\pmajmath} \Bigg(\frac{ \exp(c \xcau)}{1 + \exp(c \xcau)}\Bigg)\xcau \right] \\
  &= \frac{1}{2} \E_{\xcau \sim \sN(1, \sigcau^2)}\left[\frac{1}{\pmajmath} \Bigg(\frac{-\exp(- c \xcau)}{1 + \exp(- c \xcau)}\Bigg)\xcau \right] + \frac{1}{2} \E_{\xcau \sim \sN(-1, \sigcau^2)}\left[\frac{1}{\pmajmath} \Bigg(\frac{1}{1 + \exp(-c \xcau)}\Bigg)\xcau \right] \\
  &= \frac{1}{2 \pmajmath} \frac{1}{\sigcau \sqrt{2 \pi}}\int_{-\infty}^\infty \frac{\exp(-c \xcau) \exp \Big( \frac{- (x - 1)^2}{2 \sigcau^2}\Big) - \exp \Big( \frac{- (x + 1)^2}{2 \sigcau^2}\Big) }{1 + \exp(-c \xcau)} \xcau~d\xcau \\
  &=  \frac{1}{2 \pmajmath} \frac{1}{\sigcau \sqrt{2 \pi}}\int_{-\infty}^\infty 0~d\xcau, ~\text{Substituting $c = \frac{2}{\sigcau^2}$} \\
  &= 0.
\end{align*}
Similarly, we get $\nabla_\textsf{core}\Lrwmin(w^\star) = 0$ and hence proved that $\nabla_\textsf{core} \Lrw(w^\star) = 0$.
\end{proof}

\begin{restatable}[ ]{lemma}{cova}
  \label{lem:cov}
  The following is true.
  \begin{align}
    \cov [\nabla \ellrw(x, y, w^\star)] \preceq \diag\left(\frac{\sigcau^2 + 1}{\pmajmath(1 - \pmajmath)},\frac{\sigspu^2 + 1}{\pmajmath(1 - \pmajmath)}\right).
  \end{align}
\end{restatable}

We now compute the asymptotic variance which involves computing $\nabla^2 L(w^\star)$ and $\cov[ \nabla \ellrw(w^\star)]$.

\begin{proof}
  First, we show that the off-diagonal entries of $\cov [\ellrw(x, y, w^\star)]$ are zero.
  \begin{align*}
    &\E [ \nabla_\textsf{core} \ellrw(x, y, w^\star) \nabla_\textsf{spu} \ellrw(x, y, w^\star)] - \E [ \nabla_\textsf{core} \ellrw(x, y, w^\star)] \E [ \nabla_\textsf{spu} \ellrw(x, y, w^\star)] \\
    &= \E [ \nabla_\textsf{core} \ellrw(x, y, w^\star) \nabla_\textsf{spu} \ellrw(x, y, w^\star)] \\
    &= \pmajmath \E_y \E_{\xcau \sim \sN(y, \sigcau^2)}\E_{\xspu \sim \sN(y, \sigspu^2)} \left[\frac{1}{\pmajmath^2} \Bigg(\frac{-y \exp(-y c \xcau)}{1 + \exp(-y c \xcau)}\Bigg)^2\xcau\xspu\right] \\
    & + (1 - \pmajmath) \E_y \E_{\xcau \sim \sN(y, \sigcau^2)}\E_{\xspu \sim \sN(-y, \sigspu^2)} \left[\frac{1}{(1 - \pmajmath)^2} \Bigg(\frac{-y \exp(-y c \xcau)}{1 + \exp(-y c \xcau)}\Bigg)^2\xcau\xspu\right] \\
    &= \E_y \E_{\xcau \sim \sN(y, \sigcau^2)}\left[\frac{1}{\pmajmath} \Bigg(\frac{-y \exp(-y c \xcau)}{1 + \exp(-y c \xcau)}\Bigg)^2y\right] \\
    &- \E_y \E_{\xcau \sim \sN(y, \sigcau^2)}\left[\frac{1}{1 - \pmajmath} \Bigg(\frac{-y \exp(-y c \xcau)}{1 + \exp(-y c \xcau)}\Bigg)^2y\right] \\
    &= \frac{1-2 \pmajmath}{2\pmajmath(1 - \pmajmath)}\E_{\xcau \sim \sN(1, \sigcau^2)}\left[\Bigg(\frac{\exp(-c \xcau)}{1 + \exp(-c \xcau)}\Bigg)^2\right] - \frac{1-2 \pmajmath}{2\pmajmath(1 - \pmajmath)}\E_{\xcau \sim \sN(-1, \sigcau^2)}\left[\Bigg(\frac{\exp(c \xcau)}{1 + \exp(c \xcau)}\Bigg)^2\right] \\
    &= \frac{1-2 \pmajmath}{2\pmajmath(1 - \pmajmath)}\E_{\xcau \sim \sN(1, \sigcau^2)}\left[\Bigg(\frac{\exp(-c \xcau)}{1 + \exp(-c \xcau)}\Bigg)^2\right] - \frac{1-2 \pmajmath}{2\pmajmath(1 - \pmajmath)}\E_{\xcau \sim \sN(1, \sigcau^2)}\left[\Bigg(\frac{\exp(-c \xcau)}{1 + \exp(-c \xcau)}\Bigg)^2\right] = 0.
  \end{align*}
  Now, we bound the diagonal elements.
  \begin{align*}
    &\E [ \nabla_\textsf{core} (\ellrw(x, y, w^\star))^2] - (\E [ \nabla_\textsf{core} \ellrw(x, y, w^\star)])^2 \\
    &= \E [ \nabla_\textsf{core} (\ellrw(x, y, w^\star))^2] \\
    &= \pmajmath \E_y \E_{\xcau \sim \sN(y, \sigcau^2)}\left[\frac{1}{\pmajmath^2} \Bigg(\frac{-y \exp(-y c \xcau)}{1 + \exp(-y c \xcau)}\Bigg)^2\xcau^2 \right] \\
    &+ (1 - \pmajmath) \E_y \E_{\xcau \sim \sN(y, \sigcau^2)}\left[\frac{1}{(1 - \pmajmath)^2} \Bigg(\frac{-y \exp(-y c \xcau)}{1 + \exp(-y c \xcau)}\Bigg)^2\xcau^2 \right] \\
    &= \frac{1}{\pmajmath (1 - \pmajmath)} \E_y \E_{\xcau \sim \sN(y, \sigcau^2)}\left[\Bigg(\frac{-y \exp(-y c \xcau)}{1 + \exp(-y c \xcau)}\Bigg)^2\xcau^2 \right] \\
    &= \frac{1}{2 \pmajmath (1 - \pmajmath)} \E_{\xcau \sim \sN(1, \sigcau^2)}\left[\Bigg(\frac{- \exp(-c \xcau)}{1 + \exp(-c \xcau)}\Bigg)^2\xcau^2 \right] +  \frac{1}{2 \pmajmath (1 - \pmajmath)} \E_{\xcau \sim \sN(-1, \sigcau^2)}\left[\Bigg(\frac{- \exp(c \xcau)}{1 + \exp(c \xcau)}\Bigg)^2\xcau^2 \right] \\
    &= \frac{1}{\pmajmath (1 - \pmajmath)} \E_{\xcau \sim \sN(1, \sigcau^2)}\left[\Bigg(\frac{- \exp(-c \xcau)}{1 + \exp(-c \xcau)}\Bigg)^2\xcau^2 \right] \\
    &\leq \frac{1}{\pmajmath (1 - \pmajmath)}\E_{\xcau \sim \sN(1, \sigcau^2)} [\xcau^2] = \frac{\sigcau^2 + 1}{\pmajmath(1 - \pmajmath)}.
  \end{align*}
  Finally,
  \begin{align*}
    &\E [ \nabla_\textsf{spu} (\ellrw(x, y, w^\star))^2] - (\E [ \nabla_\textsf{spu} \ellrw(x, y, w^\star)])^2 \\
    &= \E [ \nabla_\textsf{spu} (\ellrw(x, y, w^\star))^2] \\
    &= \pmajmath \E_y \E_{\xcau \sim \sN(y, \sigcau^2)}\E_{\xspu \sim \sN(y, \sigspu^2)}\left[\frac{1}{\pmajmath^2} \Bigg(\frac{-y \exp(-y c \xcau)}{1 + \exp(-y c \xcau)}\Bigg)^2\xspu^2 \right] \\
    &+ (1 - \pmajmath) \E_y \E_{\xcau \sim \sN(y, \sigcau^2)}\E_{\xspu \sim \sN(-y, \sigspu^2)}\left[\frac{1}{(1 - \pmajmath)^2} \Bigg(\frac{-y \exp(-y c \xcau)}{1 + \exp(-y c \xcau)}\Bigg)^2\xspu^2 \right] \\
    &\leq \frac{1}{\pmajmath}\E_y \E_{\xspu \sim \sN(y, \sigspu^2)}[\xspu^2] + \frac{1}{1 - \pmajmath}\E_y \E_{\xspu \sim \sN(-y, \sigspu^2)}[\xspu^2] = \frac{\sigspu^2 + 1}{\pmajmath(1 - \pmajmath)}.
  \end{align*}
\end{proof}

\begin{restatable}[]{lemma}{secondderiv}
\label{lem:secondderiv}
  The following is true.
  \begin{align}
    \nabla^2 \Lrw(x, y, w^\star)] \succeq \diag\Bigg(\frac{\exp \Big(\frac{-4}{(\sigcau^2 + 8)\sigcau^2} \Big) (\sigcau^2 + 9)}{4 (1 + 8/\sigcau^2)^{3/2}}, \frac{\exp \Big(\frac{-4}{(\sigcau^2 + 8)\sigcau^2} \Big) (\sigspu^2 + 1)}{4 \sqrt{1 + 8/\sigcau^2}}\Bigg).
  \end{align}
\end{restatable}

\begin{proof}
 We use the following expression for computing the population gradient.
    \begin{align}
      \nabla^2 \log \Bigg(\frac{1}{1 + \exp(-yw^\top x)}\Bigg) &= \nabla \Bigg(\frac{-y \exp(- yw^\top x)}{1 + \exp(-yw^\top x)}\Bigg)x
      &= \nabla \Bigg(\frac{-y }{1 + \exp(yw^\top x)}\Bigg)x
      &= \Bigg(\frac{\exp(y w^\top x)}{(1 + \exp(yw^\top x))^2}\Bigg)xx^\top.
    \end{align}
    Recall the definition of the population majority and minority losses (Equation~\ref{eq:poplosses}).
    \begin{align}
    \nabla^2 \Lrwmaj(w)  &=  \E_y \E_{\xcau \sim \sN(y, \sigcau^2)}\E_{\xspu \sim \sN(y, \sigspu^2)} \left[\frac{1}{\pmajmath} \Bigg(\frac{\exp(yw^\top x)}{(1 + \exp(yw^\top x))^2}\Bigg)x x^\top \right]. \\
    \nabla^2 \Lrwmin(w)  &=  \E_y \E_{\xcau \sim \sN(y, \sigcau^2)}\E_{\xspu \sim \sN(-y, \sigspu^2)} \left[ \frac{1}{1 - \pmajmath}\Bigg(\frac{\exp(yw^\top x)}{(1 + \exp(yw^\top x))^2}\Bigg)x x^\top \right].
  \end{align}
  Like previously, we first compute the off-diagonal entries.
 \begin{align*}
    [\nabla^2 \Lrwmaj(w^\star)]_\text{spu, core} &= \frac{1}{\pmajmath} \E_y \E_{\xcau \sim \sN(y, \sigcau^2)}\E_{\xspu \sim \sN(y, \sigspu^2)} \left[\Bigg(\frac{\exp(y{w^\star}^\top x)}{(1 + \exp(y {w^\star}\top x))^2}\Bigg) \xcau \xspu \right] \\
    &+ \frac{1}{\pmajmath} \E_y \E_{\xcau \sim \sN(y, \sigcau^2)}\E_{\xspu \sim \sN(-y, \sigspu^2)} \left[\Bigg(\frac{\exp(y {w^\star}^\top x)}{(1 + \exp(y{w^\star}^\top x))^2}\Bigg) \xcau \xspu \right] \\
    &= \frac{1}{\pmajmath} \E_y \E_{\xcau \sim \sN(y, \sigcau^2)}\E_{\xspu \sim \sN(y, \sigspu^2)} \left[\Bigg(\frac{\exp(y {w^\star}^\top x)}{(1 + \exp(y {w^\star}^\top x))^2}\Bigg) \xcau \xspu \right] \\
    &- \frac{1}{\pmajmath} \E_y \E_{\xcau \sim \sN(y, \sigcau^2)}\E_{\xspu \sim \sN(y, \sigspu^2)} \left[\Bigg(\frac{\exp(y{w^\star}^\top x)}{(1 + \exp(y{w^\star}^\top x))^2}\Bigg) \xcau \xspu \right] \\
    &=0 \\
    [\nabla^2 \Lrwmin(w^\star)]_\text{spu, core} &= 0, ~\text{Similar calculation as above} \\
    [\nabla^2 \Lrw(w^\star)]_\text{spu, core} &=0.
\end{align*}
Now, we bound the diagonal entries. Recall that $\spu{w^\star} = 0$ and $\cau{w^\star} = c$ where $c = \frac{2}{\sigcau^2}$.
\begin{align*}
    [\nabla^2 \Lrwmaj(w^\star)]_\text{core, core} &= \frac{1}{\pmajmath} \E_y \E_{\xcau \sim \sN(y, \sigcau^2)}\left[\Bigg(\frac{ \exp(y c \xcau)}{(1 + \exp(y c \xcau))^2}\Bigg) \xcau^2 \right] \\
    &= \frac{1}{2 \pmajmath} \E_{\xcau \sim \sN(1, \sigcau^2)}\left[\Bigg(\frac{\exp(c \xcau)}{(1 + \exp(c \xcau))^2}\Bigg) \xcau^2 \right] +  \frac{1}{2\pmajmath} \E_{\xcau \sim \sN(-1, \sigcau^2)}\left[\Bigg(\frac{\exp(-c \xcau)}{(1 + \exp(-c \xcau))^2}\Bigg) \xcau^2 \right] \\
    &= \frac{1}{\pmajmath} \E_{\xcau \sim \sN(1, \sigcau^2)}\left[\Bigg(\frac{\exp(c \xcau)}{(1 + \exp(c \xcau))^2}\Bigg) \xcau^2 \right] \\
    &\geq \frac{1}{\pmajmath} \frac{1}{4}\E_{\xcau \sim \sN(1, \sigcau^2)}\left[\exp(-c^2 \xcau^2) \xcau^2\right] \\
    &= \frac{1}{\pmajmath}\frac{1}{4 \sigcau \sqrt{2 \pi}} \int_{-\infty}^\infty \exp(-c^2 \xcau^2) \exp \Big( \frac{ - (\xcau - 1)^2}{2 \sigcau^2} \Big)\xcau^2~d\xcau \\
    &= \frac{1}{\pmajmath}\frac{1}{4 \sigcau \sqrt{2 \pi}} \int_{-\infty}^\infty \exp\Big(-\frac{8 \xcau^2/\sigcau^2}{2 \sigcau^2}\Big) \exp \Big( \frac{ - (\xcau - 1)^2}{2 \sigcau^2} \Big)\xcau^2~d\xcau \\
    &= \frac{1}{\pmajmath}\frac{ \exp \Big(\frac{-8}{(\sigcau^2 + 8)\sigcau^2} \Big)}{4 \sigcau \sqrt{2 \pi}} \int_{-\infty}^\infty \exp \Big( \frac{ -(\sqrt{1 + 8/\sigcau^2}\xcau - \frac{1}{\sqrt{1 + 8/\sigcau^2}})^2}{2 \sigcau^2} \Big) \xcau^2~d\xcau \\
    &= \frac{1}{\pmajmath}\frac{\exp \Big(\frac{-8}{(\sigcau^2 + 8)\sigcau^2} \Big) (\sigcau^2 + 9)}{4 (1 + 8/\sigcau^2)^{5/2}}. \\
    [\nabla^2 \Lrwmin(w^\star)]_\text{core, core} &= \frac{1}{1 - \pmajmath}\frac{\exp \Big(\frac{-8}{(\sigcau^2 + 8)\sigcau^2} \Big) (\sigcau^2 + 9)}{4 (1 + 8/\sigcau^2)^{5/2}}, ~\text{By symmetry}. \\
    [\nabla^2 \Lrw(w^\star)]_\text{core, core} &= \pmajmath [\nabla^2 \Lrwmaj(w^\star)]_\text{core, core} + (1 - \pmajmath) [\nabla^2 \Lrwmin(w^\star)]_\text{core, core} \\
    &= \frac{\exp \Big(\frac{-8}{(\sigcau^2 + 8)\sigcau^2} \Big) (\sigcau^2 + 9)}{4 (1 + 8/\sigcau^2)^{5/2}}.
    \end{align*}

Finally, we calculate $[\nabla^2 \Lrwmaj(w^\star)]_\text{spu, spu}$ as follows.
\begin{align*}
    [\nabla^2 \Lrwmaj(w^\star)]_\text{spu, spu} &= \frac{1}{\pmajmath} \E_y \E_{\xcau \sim \sN(y, \sigcau^2)}\E_{\xspu \sim \sN(y, \sigspu^2)}\left[\Bigg(\frac{ \exp(y c \xcau)}{(1 + \exp(y c \xcau))^2}\Bigg) \xspu^2 \right] \\
    &= \frac{1}{2 \pmajmath} \E_{\xcau \sim \sN(1, \sigcau^2)}\left[\Bigg(\frac{ \exp(c \xcau)}{(1 + \exp(c \xcau))^2}\Bigg) \right](\sigspu^2 + 1) \\
    &+ \frac{1}{2 \pmajmath} \E_{\xcau \sim \sN(-1, \sigcau^2)}\left[\Bigg(\frac{ \exp(-c \xcau)}{(1 + \exp(-c \xcau))^2}\Bigg)\right] (\sigspu^2 + 1)\\
    &\geq \frac{1}{4 \pmajmath} \E_{\xcau \sim \sN(1, \sigcau^2)} [ \exp( - c^2 \xcau^2) ] (\sigspu^2 + 1) \\
    &= \frac{1}{4 \pmajmath} \frac{\exp \Big(\frac{-4}{(\sigcau^2 + 8)\sigcau^2} \Big)}{\sqrt{1 + 8/\sigcau^2}}(\sigspu^2 + 1) \\
    [\nabla^2 \Lrwmin(w^\star)]_\text{spu, spu} &= \frac{1}{4(1 - \pmajmath)}\frac{\exp \Big(\frac{-4}{(\sigcau^2 + 8)\sigcau^2} \Big)}{\sqrt{1 + 8/\sigcau^2}}(\sigspu^2 + 1), ~\text{By symmetry}. \\
    [\nabla^2 \Lrw(w^\star)]_\text{spu, spu} &= \frac{\exp \Big(\frac{-4}{(\sigcau^2 + 8)\sigcau^2} \Big) (\sigspu^2 + 1)}{4 \sqrt{1 + 8/\sigcau^2}}.
    \end{align*}
\end{proof}

\finalasymptotic*
\begin{proof}
By asymptotic normality, we have $\sqrt{n}(\hat{w} - w^\star) \rightarrow \sN(0, \nabla^2 L(w^\star)^{-1} \cov[\nabla \ell(x, y, w^\star)] \nabla^2 L(w^\star)^{-1} )$.
Combining \reflem{cov} and \reflem{secondderiv}, we get the expression in Equation~\ref{eq:asyvariance}.
Each term is decreasing in $\sigcau$, and hence we get the final result by substituting $\sigcau^2 = 1$ to obtain the constants $C_1, C_2$ (and noting that $\sigspu^2 \geq 0$).
\end{proof}

\end{document}